\documentclass[twocolumn,amsmath,nofootinbib]{revtex4} 

\usepackage{url}
\usepackage{amssymb}
\usepackage{amsfonts}
\usepackage{amsmath}
\usepackage{amsthm}
\usepackage{amsbsy}
\usepackage{graphicx}
\usepackage{hyperref}
\usepackage{array} 
\usepackage{multirow}
\usepackage{xcolor}

\newcolumntype{x}[1]{%
>{\centering\hspace{0pt}}p{#1}}%
\providecommand{\openone}{\leavevmode\hbox{\small1\kern-3.8pt\normalsize1}}

\def\ie{{\frenchspacing\it i.e.}}
\def\eg{{\frenchspacing\it e.g.}}
\def\etc{{\frenchspacing\it etc.}}

\def\spose#1{\hbox to 0pt{#1\hss}}
\def\simlt{\mathrel{\spose{\lower 3pt\hbox{$\mathchar"218$}}
   \raise 2.0pt\hbox{$\mathchar"13C$}}}
\def\simgt{\mathrel{\spose{\lower 3pt\hbox{$\mathchar"218$}}
     \raise 2.0pt\hbox{$\mathchar"13E$}}}
 \def\simpropto{\mathrel{\spose{\lower 3pt\hbox{$\mathchar"218$}}
     \raise 2.0pt\hbox{$\propto$}}}

\def\beq#1{\begin{equation}\label{#1}}
\def\eeq{\end{equation}}
\def\beqa#1{\begin{eqnarray}\label{#1}}
\def\eeqa{\end{eqnarray}}
\def\eq#1{equation~(\ref{#1})}	

\def\eqn#1{~(\ref{#1})}

\def\fig#1{Figure~\ref{#1}}
\def\Fig#1{Figure~\ref{#1}}

\def\Sec#1{Section~\ref{#1}}

\def\ed{\end{document}}

\def\a{{\bf a}}
\def\b{{\bf b}}

\def\dKL{d_{\rm{KL}}}

\def\Ell{{\mathcal L}}

\def\pbar{\bar{p}}

\def\u{{\bf u}}

\def\what{{\hat w}}
\def\x{{\bf x}}
\def\y{{\bf y}}
\def\z{{\bf z}}

\def\B{{\bf B}}

\def\F{{\bf F}}
\def\G{{\bf G}}

\def\P{{\bf P}}

\def\expec#1{\langle#1\rangle}

\def\rn{}
\def\nn#1 #2{#2. #1}				
\def\nnn#1 #2 #3{#2. #3. #1}			
\def\nnnn#1 #2 #3 #4{#2. #3. #4 #1}		
\def\nnnnn#1 #2 #3 #4 #5{#2. #3. #4 #5. #1}	

\def\rf#1;#2;#3;#4;#5 {{\frenchspacing\par\rn#1, #3 {\bf #4}, #5 (#2). \par}}
\def\rg#1;#2;#3;#4;#5;#6 {{\frenchspacing\par\rn#1, #3 {\bf #4}, #5 (#2). \par}}
\def\rfbook#1;#2;#3;#4;#5 {{\frenchspacing\par\rn#1, {\it #3} (#5, #4, #2).\par}}
\def\rfprep#1;#2;#3 {{\par\frenchspacing\rn#1, #3 (#2).\par}}
\def\rfproc#1;#2;#3;#4;#5;#6 {{\frenchspacing\par\rn#1 #2, in {\it #3}, ed. #4 (#5: #6)\par}}
\def\rfprocp#1;#2;#3;#4;#5;#6;#7 {{\frenchspacing\par\rn#1 #2, in {\it #3}, ed. #4 (#5: #6), p#7\par}}

\newtheorem{theorem}{Theorem}[section]
\newtheorem{corollary}{Corollary}[theorem]

\begin{document}
\pdfoptionalwaysusepdfpagebox=5


\title{Pareto-optimal data compression for binary classification tasks}

\author{Max Tegmark \& Tailin Wu}

\address{Dept.~of Physics, MIT Kavli Institute \& Center for Brains, Minds \& Machines, Massachusetts Institute of Technology, Cambridge, MA 02139; tegmark@mit.edu}

\date{ December 19, 2019; published in {\it Entropy}, {\bf 22}, 7}


\begin{abstract}
The goal of lossy data compression is to reduce the storage cost of a data set $X$ while retaining as much information as possible about something ($Y$) that you care about. For example, what aspects of an image $X$ contain the most information about whether it depicts a cat?
Mathematically, this corresponds to finding a deterministic mapping $X\to Z\equiv f(X)$ that 
maximizes the mutual information $I(Z,Y)$ while the entropy $H(Z)$ is kept below some fixed threshold. 
We present a new method for mapping out the Pareto frontier for classification tasks, reflecting the tradeoff between retained entropy and class information. We first show how a random variable $X$ (an image, say) drawn from a class $Y\in\{1,...,n\}$ can be distilled into a vector $W=f(X)\in \mathbb{R}^{n-1}$ losslessly, so that $I(W,Y)=I(X,Y)$; for example, for a binary classification task of cats and dogs, each image $X$ is mapped into a single real number $W$ retaining all information that helps distinguish cats from dogs. For the $n=2$ case of binary classification, we then show how $W$ can be further compressed into a discrete variable $Z=g_\beta(W)\in\{1,...,m_\beta\}$ by binning $W$ into $m_\beta$ bins, in such a way that varying the parameter $\beta$ sweeps out the full Pareto frontier, solving a generalization of the Discrete Information Bottleneck (DIB) problem.
We argue that the most interesting points on this frontier are ``corners" maximizing $I(Z,Y)$ for a fixed number of bins $m=2,3...$ which can be conveniently be found without multiobjective optimization. We apply this method to the CIFAR-10, MNIST and Fashion-MNIST datasets, illustrating how it can be interpreted as an information-theoretically optimal image clustering algorithm. We find that these Pareto frontiers are not concave, and that recently reported DIB phase transitions correspond to transitions between these corners, changing the number of clusters.
\end{abstract}

\maketitle


\section{Introduction}

A core challenge in science, and in life quite generally, is data distillation: 
keeping only a manageably small fraction of our available data $X$ while retaining as much information as possible about something ($Y$) that we care about. For example, what aspects of an image contain the most information about whether it depicts a cat ($Y=1$) rather than a dog ($Y=2$)?
Mathematically, this motivates finding a mapping $X\to Z\equiv g(X)$ that 
maximizes the mutual information $I(Z,Y)$ while the entropy $H(Z)$ is kept below some fixed threshold. 
The tradeoff between $H_*=H(Z)$ (bits stored) and $I_*=I(Z,Y)$ (useful bits) is described by a Pareto frontier, defined as  
\beq{ParetoDefEq}
I_*(H_*) \equiv\sup_{\{g: H[g(X)]\le H_*\}} I[g(X),Y],
\eeq
and illustrated in \fig{paretoAnalyticFig} (this is for a toy example described below; we compute the Pareto frontier for our cat/dog example in \Sec{ResultsSec}).
The shaded region is impossible because 
$I(Z,Y)\le I(X,Y)$ 
and  $I(Z,Y)\le H(Z)$.
The colored dots correspond to random likelihood binnings into various numbers of bins, as described in the next section, and the upper envelope of all attainable points define the Pareto frontier. Its ``corners'', which are marked by black dots and maximize $I(Z,Y)$ for $M$ bins ($M=1,2,...$), are seen to lie close to the vertical dashed lines 
$H(Z)=\log M$, corresponding to all bins having equal size. We plot the $H$-axis flipped to conform with the tradition that up and to the right are more desirable.
The core goal of this paper is to present a method for computing such Pareto frontiers.

\begin{figure}[h!]
\begin{center}
\vskip-5.4mm
\includegraphics[width=\columnwidth]{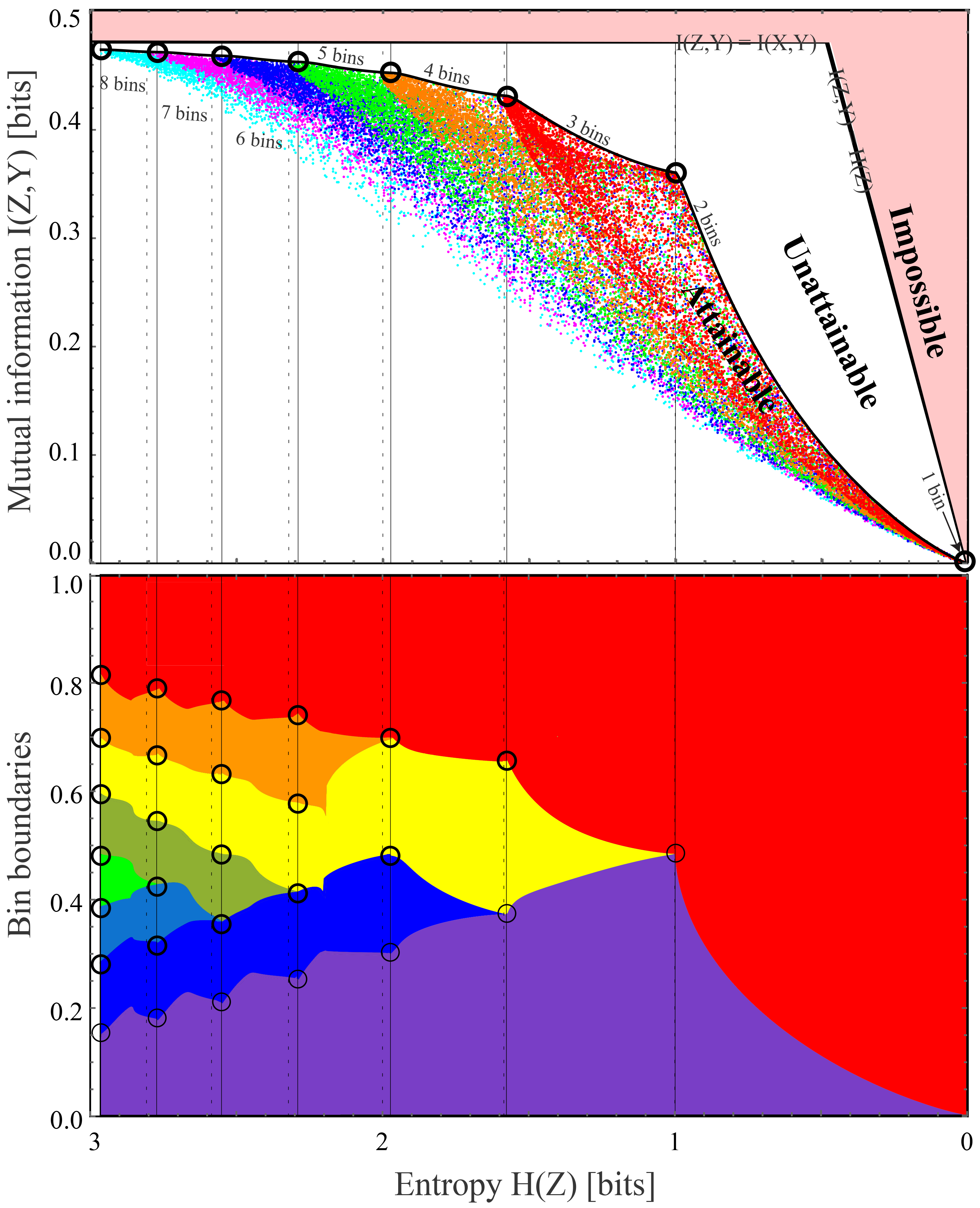}
\end{center}
\vskip-8mm
\caption{The Pareto frontier (top panel) for compressed versions $Z=g(X)$ of our warmup dataset $X)\in[0,1]^2$ with classes $Y\in\{1,2\}$, showing the maximum attainable class information $I(Z,Y)$ for a given entropy $H(Z)$, mapped using the method described in this paper using the likelihood binning in the bottom panel.
}
\label{paretoAnalyticFig}
\end{figure}

\subsection{Objectives \& relation to prior work}

In other words, the goal of this paper is to analyze soft rather than hard classifiers:
not to make the most accurate classifier, but rather to compute the Pareto frontier that reveals the 
most accurate (in an information-theoretic sense) classifier $Z$ given a constraint on its bit content $H(Z)$.
These optimal soft classifiers that we will derive (corresponding to points on the Pareto frontier) are useful
for the same reason that other methods for lossy data compression methods are useful:
overfitting less and therefore generalizing better, among other things.

This Pareto frontier challenge is thus part of the broader quest for data distillation:
lossy data compression that retains as much as possible of the information that is useful to us. 
Ideally, the information can be partitioned into a set of independent chunks and sorted from most to least useful, enabling us to select the number of chunks to retain so as to optimize our tradeoff between utility and data size.
Consider two random variables $X$ and $Y$ which may each be vectors or scalars.
For simplicity, consider them to be discrete with finite entropy\footnote{The discreteness restriction loses us no generality in practice, since  since we can always discretize real numbers by rounding them to some very large number of significant digits.}.
For prediction tasks, we might interpret $Y$ as the future state of a dynamical system that we wish to predict from the present state $X$. For classification tasks, we might interpret $Y$ as a class label that we wish to predict from an image, sound, video or text string $X$.  Let us now consider various forms of ideal data distillation, as summarized in Table~\ref{ComparisonTable}.

\begin{table}[h!]
\begin{tabular}{|c|l|l|l|}
\hline
Random		&What is				&\multicolumn{2}{c|}{Probability distribution}\\
\cline{3-4}
vectors		&distilled?			&Gaussian			&Non-Gaussian\\
\hline
1			&Entropy				&PCA				&Autoencoder\\
			&$H(X)=\sum_i H(Z_i)$	&$\z=\F\x$			&$Z=f(X)$\\
\hline
2			&Mutual information		&CCA				&Latent reps\\
			&$I(X,Y)=\sum_i I(Z_i,Z'_i)$	&$\z=\F\x$			&$Z=f(X)$\\
			&					&$\z'=\G\y$			&$Z'=g(Y)$\\
\hline
\end{tabular}
\caption{Data distillation: the relationship between Principal Component Analysis (PCA), Canonical Correlation Analysis (CCA), nonlinear autoencoders and nonlinear latent representations.
\label{ComparisonTable}
}
\end{table}

If we distill $X$ as a whole, then we would ideally like to find a function $f$ such that
the so-called latent representation $Z=f(X)$ retains the full entropy
$H(X)=H(Z)=\sum H(Z_i)$, decomposed into independent\footnote{When implementing any distillation algorithm in practice, there is always a one-parameter tradeoff between compression and information retention which defines a Pareto frontier. A key advantage of the latent variables (or variable pairs) being statistically independent is that this allows the Pareto frontier to be trivially computed, by simply sorting them by decreasing information content and varying the number retained.
}
parts with vanishing mutual infomation: 
$I(Z_i,Z_j)=\delta_{ij}H(Z_i).$
For the special case where $X=\x$ is a vector with a multivariate Gaussian distribution, 
the optimal solution is 
Principal Component Analysis (PCA) \cite{pearsonPCA1901}, which
has long been a workhorse of statistical physics and many other disciplines: here $f$ is simply a linear function mapping into the eigenbasis of the covariance matrix of $\x$.
The general case remains unsolved, and it is easy to see that it is hard: 
if $X=c(Z)$ where $c$ implements some state-of-the-art cryptographic code, then finding $f=c^{-1}$ (to recover the independent pieces of information and discard the useless parts) would generically require breaking the code. Great progress has nonetheless been made for many special cases, using techniques such as nonlinear autoencoders \cite{vincent2008extracting} and Generative Adversarial Networks (GANs) \cite{goodfellow2014generative}.

Now consider the case where we wish to distill $X$ and $Y$ separately,
into $Z\equiv f(X)$ and $Z'=g(Y)$, retaining the mutual information between the two parts. 
Then we ideally have
$I(X,Y)=\sum_i I(Z_i,Z'_i)$,
$I(Z_i,Z_j)=\delta_{ij}H(Z_i),$
$I(Z'_i,Z'_j)=\delta_{ij}H(Z'_i),$
$I(Z_i,Z'_j)=\delta_{ij}I(Z_i,Z'_j).$
This problem has attracted great interest, especially for time series where $X=\u_i$ and $Y=\u_j$ for some sequence of states $\u_k$ ($k=0,1, 2, ...$) in physics or other fields, 
where one typically maps the state vectors $\u_i$ into some lower-dimensional 
vectors $f(\u_i)$, after which the prediction is carried out in this latent space. 
For the special case where $X$ has a multivariate Gaussian distribution, 
the optimal solution is 
Canonical Correlation Analysis (CCA) \cite{hotellingCCA1936}: here both $f$ and $g$ are linear functions, computed via a  
singular-value decomposition (SVD) \cite{eckart1936SVD} of the cross-correlation matrix after prewhitening $X$ and $Y$.
The general case remains unsolved, and is obviously even harder than the above-mentioned 1-vector autoencoding problem. 
The recent work \cite{oord2018representation,clark2019unsupervised} review the state-of-the art as well as presenting  Contrastive Predictive Coding and Dynamic Component Analysis, powerful new distillation techniques for time series, following the long tradition of setting $f=g$ even though this is provably not optimal for the Gaussian case as shown in \cite{tegmark2019optimal}.

The goal of this paper is to make progress in the lower right quadrant of Table~\ref{ComparisonTable}. 
We will first show that if $Y\in \{1,2\}$ (as in binary classification tasks) and we can successfully train a classifier that correctly predicts the conditional probability distribution $p(Y|X)$, then it can be used to provide an exact solution to the distillation problem,
losslessly distilling $X$ into a single real variable $W=f(X)$. 
We will generalize this to an arbitrary classification problem $Y\in\{1,...,n\}$ by losslessly distilling $X$ into a vector $W=f(X)\in \mathbb{R}^{n-1}$, although in this case, the components of the vector $W$ may not be independent.
We will then return to the binary classification case and provide a family of binnings that map $W$ into an integer $Z$, allowing us to scan the full Pareto frontier reflecting the tradeoff between retained entropy and class information, illustrating the end-to-end procedure with the CIFAR-10, MNIST and Fashion-MNIST datasets. 
This is related to the work of \cite{kurkoski2014quantization} which maximizes $I(Z,Y)$ for a fixed number of bins (instead of for a fixed entropy), which corresponds to the ``corners'' seen in \fig{paretoAnalyticFig}.


This work is closely related to the Information Bottleneck (IB) method \cite{tishby2000information}, which provides an insightful, principled approach for balancing compression against prediction \cite{tan2019renormalization}. 
Just as in our work, the IB method aims to find a random variable $Z=f(X)$ that loosely speaking retains as much information as possible about $Y$ and as little other information as possible.
The IB method implements this by maximizing the IB-objective
\beq{IBeq}
{\cal L}_{\rm IB}= I(Z,Y)-\beta I(Z,X)
\eeq
where the Lagrange multiplier $\beta$ tunes the balance between knowing about $Y$ and forgetting about $X$. 
\cite{strouse2017deterministic} considered the alternative Deterministic Information Bottleneck (DIB) objective 
\beq{IBeq2}
{\cal L}_{\rm DIB}= I(Z,Y)-\beta H(Z),
\eeq
to close the loophole where $Z$ retains random information that is independent of both $X$ and $Y$ (which is possible if $f$ is function that contains random components rather than fully deterministic\footnote{If $Z=f(X)$ for some deterministic function $f$, 
which is typically not the case in the popular 
variational IB-implementation \cite{alemi2016deep,chalk2016relevant,fischer2018the},
then $H(Z|X)=0$, so $I(Z,X) \equiv H(Z)-H(Z|X) = H(Z)$, 
which means the two objectives\eqn{IBeq} 
and\eqn{IBeq2} are identical.}).
However, there is a well-known problem with this objective that occurs when $Z\in \mathbb{R}^n$ is continuous \cite{amjad2019learning}: 
$H(Z)$ is strictly speaking infinite, since it requires an infinite amount of information to store the infinitely many decimals of a generic real number. 
Although this infinity is normally regularized away by only defining $H(Z)$ up to an additive constant, which is irrelevant when minimizing \eqn{IBeq2}, the problem is that we can define a new rescaled random variable
\beq{RescalingEq}
Z'=aZ
\eeq
for a constant $a\ne 0$ and obtain\footnote{Throughout this paper, we take $\log$ to denote the logarithm in base $2$, so that entropy and mutual information are measured in bits.}
\beq{IscalingEq}
I(Z',X)=I(Z,X)
\eeq
and
\beq{HscalingEq}
H(Z')=H(Z)+n\log|a|.
\eeq
This means that by choosing $|a|\ll 1$, we can make $H(Z')$ arbitrarily negative while keeping $I(Z',X)$ unchanged, thus making ${\cal L}_{\rm DIB}$ arbitrarily negative.
The objective ${\cal L}_{\rm DIB}$ is therefore not bounded from below, and trying to minimize it will not produce an interesting result.
We will eliminate this $Z$-rescaling problem by making $Z$ discrete rather than continuous, so that $H(Z)$ is always well-defined and finite. 
Another challenge with the DIB objective of \eq{IBeq2}, which we will also overcome, is that it maximizes a linear combination of the two axes in \fig{paretoAnalyticFig}, and can therefore  only discover concave parts of the Pareto frontier, not convex ones (which are seen to dominate in \fig{paretoAnalyticFig}).

The rest of this paper is organized as follows:
In \Sec{MethodSecA}, we will provide an exact solution for the binary classification problem where $Y\in\{1,2\}$ by losslessly distilling $X$ into a single real variable $Z=f(X)$. We also generalize this to an arbitrary classification problem  $Y\in\{1,...,n\}$ by losslessly distilling $X$ into a vector $W=f(X)\in \mathbb{R}^{n-1}$, although the components of the vector $W$ may not be independent.
In \Sec{MethodSecB}, we return to the binary classification case and provide a family a binnings that map $Z$ into an integer, allowing us to scan the full Pareto frontier reflecting the tradeoff between retained entropy and class information. We apply our method to various image datasets in \Sec{ResultsSec} and discuss our conclusions in \Sec{ConclusionsSec}

\section{Method}
\label{LatentSec}



Our algorithm for mapping the Pareto frontier transforms our original data set $X$ in a series of steps which will be describe in turn below:
\beq{MappingsEq}
X\overset{w}{\mapsto} W\mapsto W_{\rm uniform}\mapsto W_{\rm binned}\mapsto W_{\rm sorted}\overset{B}{\mapsto} Z.
\eeq
As we will show, the first, second and fourth transformations retain all mutual information with the label $Y$, and the information loss about $Y$ can be kept arbitrarily small in the third step. In contrast, the last step treats the information loss as a tuneable parameter that parameterizes the Pareto frontier.

\subsection{Lossless distillation for classification tasks}
\label{MethodSecA}

Our first step is to compress $X$ (an image, say) into $W$, a set of $n-1$ real numbers, in such a way that no class information is lost about 
$Y\in\{1,...,n\}$. 

\begin{theorem}
\label{Wtheorem}
{\bf (Lossless Distillation Theorem):} For an arbitrary random variable $X$ and a categorical random variable $Y\in\{1,...,n\}$,
we have 
\beq{Theorem1Eq}
P(Y|X) = P(Y|W),
\eeq
where $W\equiv w(X)\in \mathbb{R}^{n-1}$
is defined by\footnote{Note that we ignore the $n^{\rm th}$ component since it is redundant:
$w_n(X)=1-\sum_i^{n-1} w_i(X)$.}
\beq{wDefEq}
w_i(X) \equiv P(Y=i | X).
\eeq
\end{theorem}

\begin{proof}
Let $S$ denote the domain of $X$, \ie, $X\in S$, and define the set-valued function 
$$s(W)\equiv \{x\in S: w(x)=W\}.$$
These sets $s(W)$ form a partition of $S$ parameterized by $W$, since they are disjoint and 
\beq{UnionEq}
\cup_{W\in\mathbb{R}^{n-1}} \>s(W) = S.
\eeq
For example, if $S=\mathbb{R}^2$ and $n=2$, then the sets $s(W)$ are simply contour curves of the 
conditional probability $W\equiv P(Y=1|X)\in \mathbb{R}$.
This partition enables us to uniquely specify $X$ as  the pair $\{W,X_W\}$ by first specifying which set $s[f(X)]$ it belongs to
(determined by $W=f(X)$), and then specifying the particular element within that set, which we denote
$X_W\in S(W)$.
This implies that 
\beq{Proof1Eq}
P(Y|X) = P(Y | W,X_W) = P(Y|W),
\eeq
completing the proof. 
The last equal sign follows from the fact that the conditional probability $P(Y|X)$ is independent of $X_W$, since it is by definition constant throughout the set $s(W)$.
\end{proof}

The following corollary implies that $W$ is an optimal distillation of the information $X$ has about $Y$, in the sense that it constitutes a lossless compression of said information:
$I(W,Y)=I(X,Y)$ as shown, and the total information content (entropy) in $W$ cannot exceed that of $X$ since it is a deterministic function thereof.
\begin{corollary}
With the same notation as above, we have
\beq{Theorem2Eq}
I(X,Y) = I(W,Y).
\eeq
\end{corollary}
\begin{proof}
For any two random variables, we have the identity $I(U,V)=H(V)-H(V|U)$, where 
 $I(U,V)$ is their mutual information and $H(V|U)$ denotes conditional entropy.
We thus obtain 
\beqa{Proof2Eq}
I(X,Y)&=&H(Y)-H(Y|X) = H(Y)+\expec{\log P(Y|X)}_{X,Y}\nonumber\\
	&=&H(Y)+\expec{\log P(Y|W)}_{W,X_W,Y}\nonumber\\
	&=&H(Y)+\expec{\log P(Y|W)}_{W,Y} \nonumber\\
	&=&H(Y)-H(Y|W) =I(W,Y),
\eeqa
which completes the proof.
We obtain the second line by using $P(Y|X) = P(Y|W)$ from Theorem~1 and specifying $X$ by $W$ and $X_W$, and the third line since $P(Y|W)$ is independent of $X_W$, as above.
\end{proof}

In most situations of practical interest, the conditional probability distribution $P(Y|X)$ is not precisely known, but can be approximated by training a neural-network-based classifier that outputs the probability distribution for $Y$ given any input $X$. We present such examples in \Sec{ResultsSec}. 
The better the classifier, the smaller the information loss $I(X,Y)-I(W,Y)$ will be, approaching zero in the limit of an optimal classifier.

\subsection{Pareto-optimal compression for binary classification tasks}
\label{MethodSecB}

 \begin{figure}[ht!]
\begin{center}
\includegraphics[width=\columnwidth]{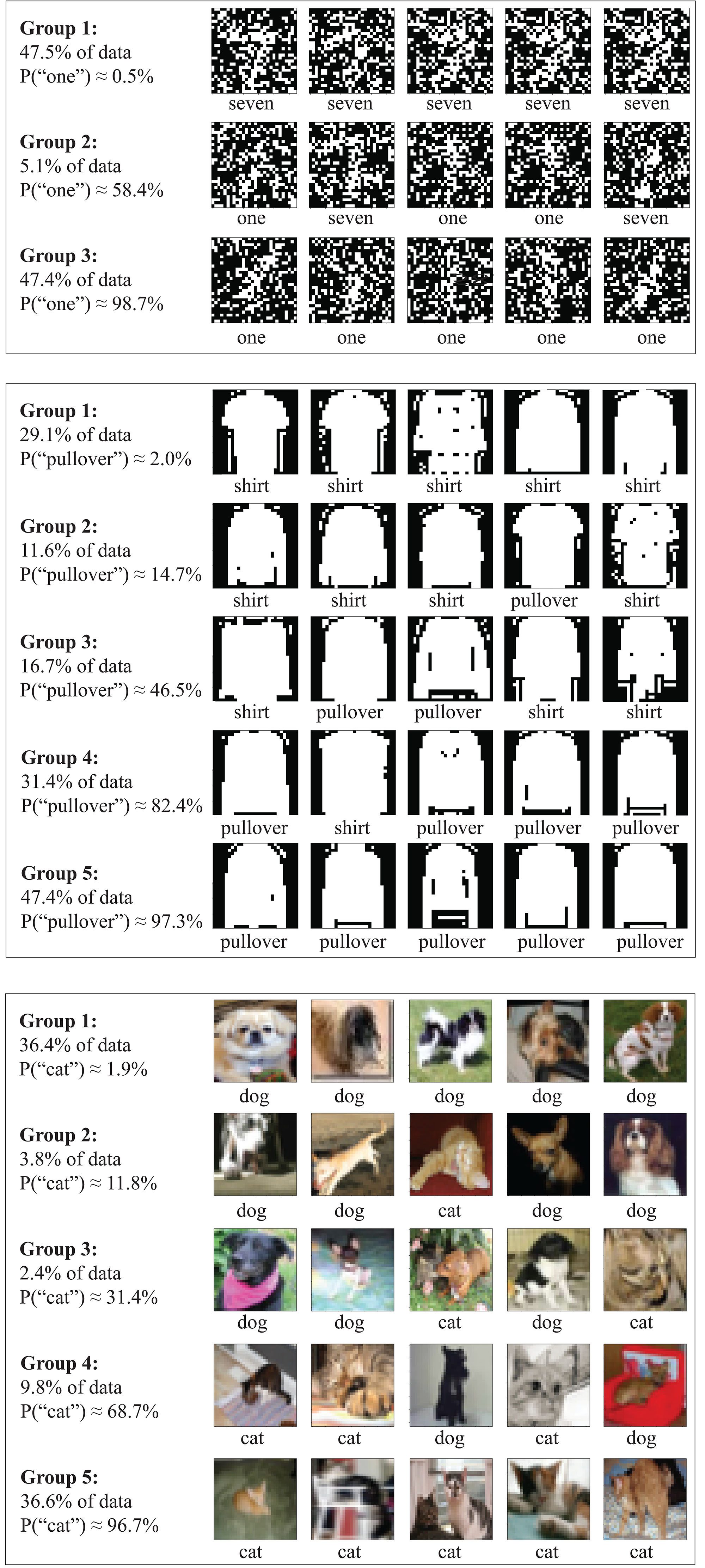}
\end{center}
\vskip-5mm
\caption{Sample data from \protect\Sec{ResultsSec}.
Images from MMNIST (top), Fashion-MNIST (middle) and CIFAR-10 are 
mapped into integers (group labels) $Z=f(X)$ retaining maximum mutual information with the class variable $Y$ (ones/sevens, shirts/pullovers and cats/dogs, respectively) for 
3, 5 and 5 groups, respectively. These mappings $f$ correspond to Pareto frontier ``corners". 
}
\label{classesFig}
\end{figure}

Let us now focus on the special case where $n=2$, \ie, binary classification tasks. 
For example, $X$ may correspond to images of equal numbers of felines and canines to be classified despite challenges with variable lighting, occlusion, {\etc} as in \fig{classesFig}, and $Y\in\{1,2\}$ may correspond to the labels ``cat" and ``dog".
In this case, $Y$ contains $H(Y)=1$ bit of information of which $I(X,Y)\le 1$ bit is contained in $X$.
Theorem~\ref{Wtheorem} shows that for this case, all of this 
information about whether an image contains a cat or a dog can be compressed into a single number $W$ which is not a bit like $Y$, but a real number between zero and one. 

The goal of this section is find a class of functions $g$ that perform Pareto-optimal lossy compression of $W$, mapping it into an 
integer $Z\equiv g(W)$ that maximizes $I(Z,Y)$ for a fixed entropy $H(Z)$.\footnote{Throughout this paper, we will use the term ``Pareto-optimal" or ``optimal" in this sense, \ie, maximizing $I(X,Y)$ for a fixed $H(Z)$.}
The only input we need for our work in this section is the joint probability distribution $f_i(w)=P(Y\hbox{=}i,W\hbox{=}w)$, whose marginal distributions are the discrete probability distribution for $P^Y_i$ for
$Y$ and the probability distribution $f$ for $W$, which we will henceforth assume to be continuous:
\beqa{WYmargEq1}
f(w)&\equiv&\sum_{i=1}^2 f_i(w),\\
P^Y_i\equiv P(Y\hbox{=}i)&=&\int_0^1 f_i(w)dw.
\eeqa

\subsubsection{Uniformization of $W$}

For convenience and without loss of generality, we will henceforth assume that $f(w)=1$, \ie, that $W$ has a uniform distribution on the unit interval $[0,1]$. We can do this because if $W$ were not uniformly distributed, we could make it so by using the standard statistical technique of applying its cumulative probability distribution function to it:
\beq{UniformizationEq}
W \mapsto W' \equiv F(W),\quad F(w)\equiv\int_0^w f(w')dw',
\eeq
retaining all information ---  $I(W',Y)=I(W,Y)$ ---
since this procedure is invertible almost everywhere.

\begin{figure}[ht]
\begin{center}
\includegraphics[width=\columnwidth]{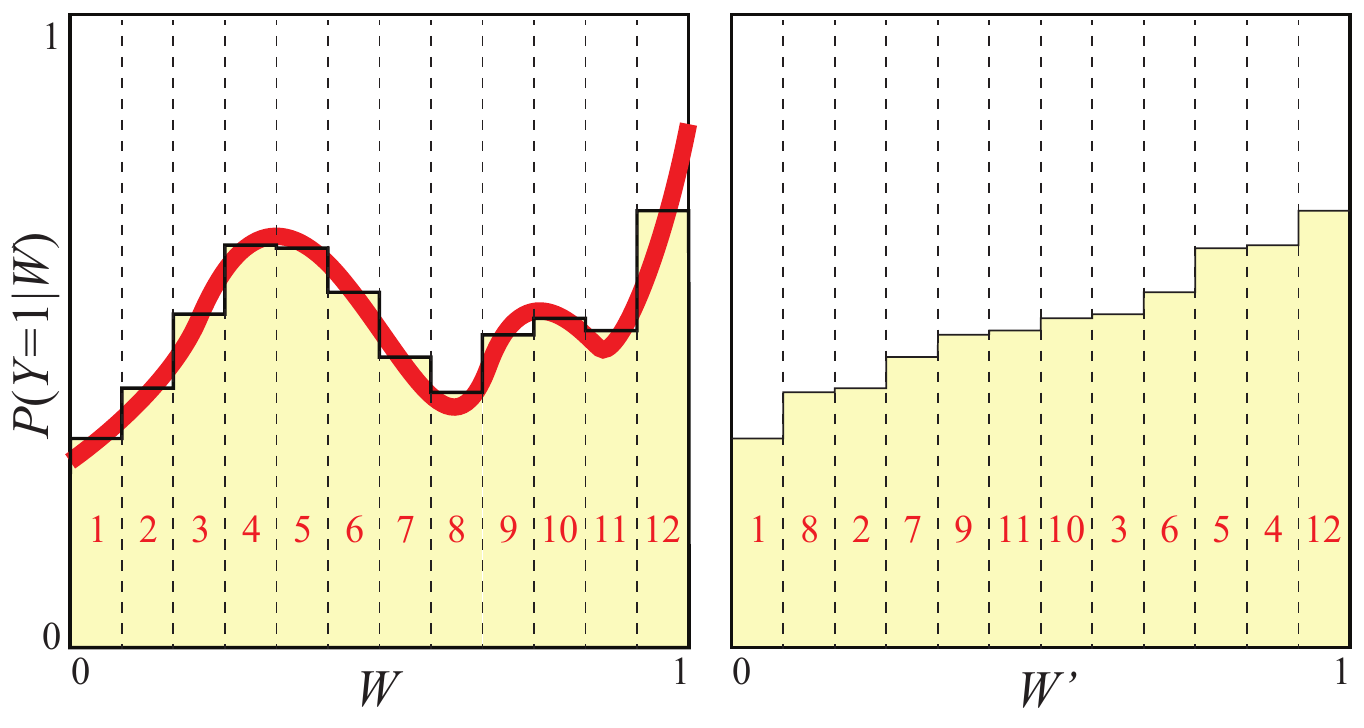}
\end{center}
\caption{Essentially all information about $Y$ is retained if $W$ is binned into sufficiently narrow bins.
Sorting the bins (left) to make the conditional probability monotonically increasing (right) changes neither this information not the entropy.
}
\label{fineBinningFig}
\end{figure}

\subsubsection{Binning $W$}
Given a set of bin boundaries $b_1<b_2<...<b_{n-1}$ grouped into a vector $\b$,
we define the integer-value {\it contiguous binning function} 
\beq{BinningFuncDef}
B(x,\b)\equiv
\left\{
\begin{tabular}{l}
$1$\quad if $x<b_1$\\
$k$\quad if $b_{k-1}<x\le b_k$\\
$n$\quad if $x\ge b_{N-1}$\\
\end{tabular}
\right.
\eeq
$B(x,\b)$ can thus be interpreted as the ID of the bin into which $x$ falls. 
Note that $B$ is a monotonically increasing piecewise constant function of $x$ that is shaped like an $N$-level staircase with $n-1$ steps at $b_1,...,b_{N-1}$. 

Let us now bin $W$ into $N$ equispaced bins, by
mapping it into an integer $W'\in\{1,...,N\}$ (the bin ID) defined by 
\beq{WprimeDefEq}
W'\equiv W_{\rm binned}\equiv B(W,\b_N),
\eeq
where $\b$ is the vector with elements $b_j=j/N$, $j=1,...,N-1$.
As illustrated visually in \fig{fineBinningFig} and mathematically in Appendix~\ref{LosslessBinningAppendix}, binning $W\mapsto W'$ corresponds
to creating a new random variable for which the conditional distribution
$p_1(w)=P(Y\hbox{=}1|W\hbox{=}w)$ is replaced by a piecewise constant function $\pbar_1(w)$, 
replacing the values in each bin by their average.
The binned variable $W'$ thus retains only information about which bin $W$ falls into, discarding all information about the precise location within that bin.
In the $N\to\infty$ limit of infinitesimal bins,  
$\pbar_1(w)\to p_1(w)$, and we expect the above-mentioned discarded information to become negligible.
This intuition is formalized by 
\ref{LosslessBinningTheorem} in Appendix~\ref{LosslessBinningAppendix}, which under mild smoothness assumptions ensuring that $p_1(w)$ is not pathological 
shows that  
\beq{binningLimitEq}
I(W',Y) \to I(W,Y)\quad\hbox{as}\quad N\to\infty,
\eeq
\ie, that we can make the binned data $W'$ retain essentially all the class information from 
$W$ as long as we use enough bins. 

In practice, such as for the numerical experiments that we will present in Section~\ref{ResultsSec}, training data is never infinite and the conditional probability function $p_1(w)$ is never known to perfect accuracy.
This means that the pedantic distinction between 
$I(W',Y)=I(W,Y)$ and $I(W',Y)\approx I(W,Y)$ for very large $N$ is completely irrelevant in practice. In the rest of this paper, we will therefore work with the
unbinned ($W$) and binned ($W'$) data somewhat interchangeably below for convenience, occasionally dropping the apostrophy $'$ from $W'$ when no confusion is caused. 

\subsubsection{Making the conditional probability monotonic}

For convenience and without loss of generality, we can assume that the conditional probability distribution $\pbar_1(w)$ is a monotonically increasing function.
We can do this because if this were not the case, we could make it so by sorting the bins by increasing conditional probability, as illustrated in 
\fig{fineBinningFig}, because both the entropy $H(W')$ and the mutual information $I(W',Y)$ are left invariant by this renumbering/relabeling of the bins. The ``cat" probability $P(Y\hbox{=}1)$ (the total shaded area in \fig{fineBinningFig}) is of course also left unchanged by both this sorting and by the above-mentioned binning.

\subsubsection{Proof that Pareto frontier is spanned by contiguous binnings}

We are now finally ready to tackle the core goal of this paper: mapping the Pareto frontier $(H_*,I_*)$ of optimal data compression $X\mapsto Z$ that reflects the tradeoff between $H(Z)$ and $I(Z,Y)$.
While fine-grained binning has no effect on the entropy $H(Y)$ and negligible effect on $I(W,Y)$, it dramatically reduces the entropy of our data.
Whereas $H(W)=\infty$ since $W$ is continuous\footnote{While this infinity, which reflects the infinite number of bits required to describe a single generic real number, is customarily eliminated by defining  entropy only up to an overall additive constant, we will not follow that custom here, for the reason explained in the introduction.}, $H(W')=\log N$ is finite, approaching infinity only in the limit of infinitely many infinitesimal bins.
Taken together, these scalings of $I$ and $H$ imply that the leftmost part of the Pareto frontier $I_*(H_*)$, defined by \eq{ParetoDefEq} and illustrated in \fig{paretoAnalyticFig}, asymptotes  to a horizontal line of height $I_*=I(X,Y)$ as $H_*\to\infty$.

To reach the interesting parts of the Pareto frontier further to the right, we must destroy some information
about $Y$. We do this by defining
\beq{ZdefEq}
Z = g(W'),
\eeq
where the function $g$ groups the tiny bins indexed by $W' \in\{1,...,N\}$ into fewer ones indexed by
$Z\in \{1,...,M\}$, $M<N$. 
There are vast numbers of such possible groupings, since each group corresponds to one of the
$2^N-2$ nontrivial subsets of the tiny bins. Fortunately, as we will now prove, we need only consider 
the $\mathcal{O}(N^M)$ {\it contiguous} groupings, since non-contiguous ones are inferior and cannot lie on the Pareto frontier. Indeed, we will see that for the examples in \Sec{ResultsSec}, $M\simlt 5$ 
suffices to capture the most interesting information. 

\begin{figure}[ht]
\begin{center}
\includegraphics[width=\columnwidth]{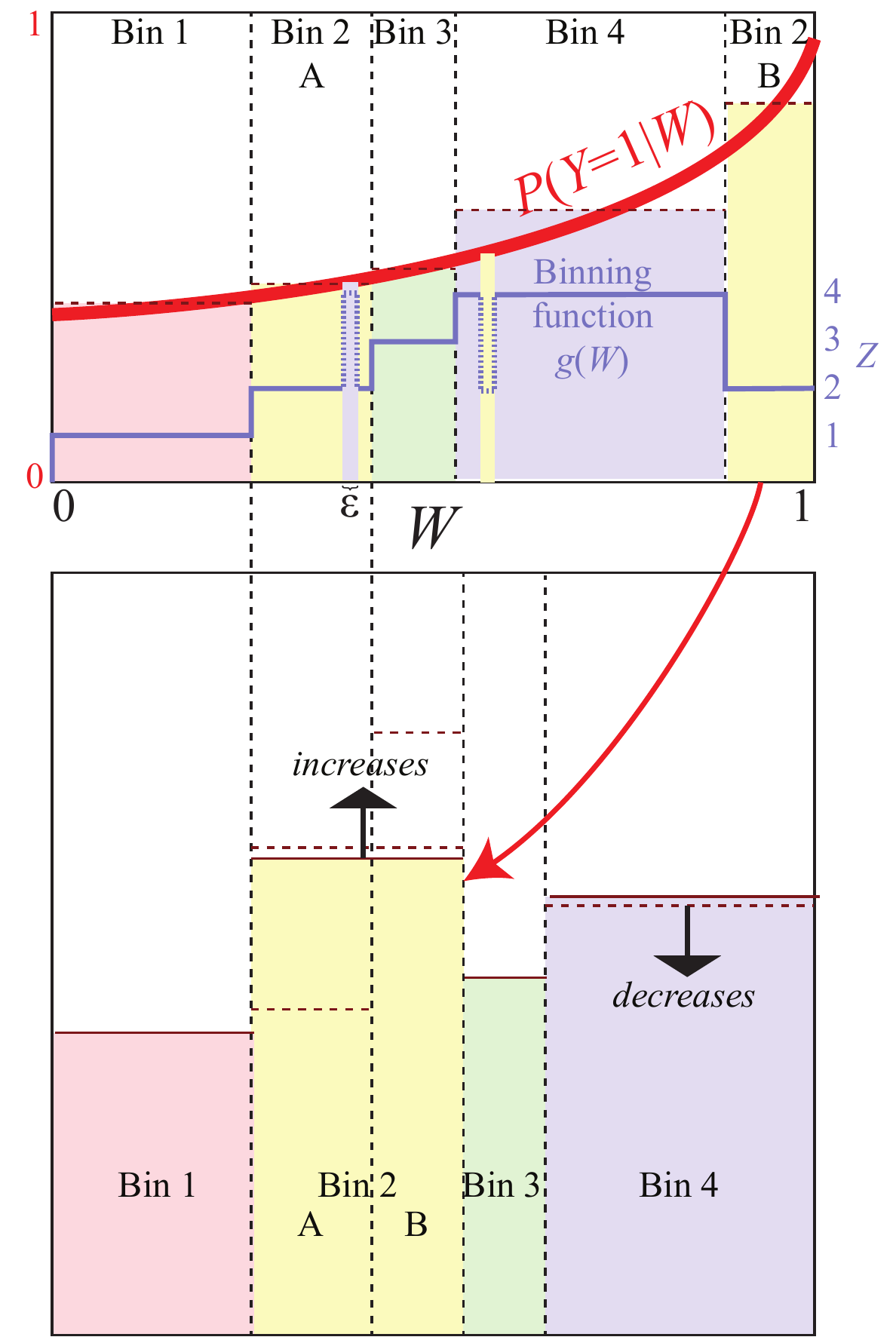}
\end{center}
\caption{The reason that the Pareto frontier can never be reached using non-contiguous bins is that a swapping 
parts of them against parts of an intermediate bin can increase $I(Z,X)$ while keeping $H(Z)$ constant. 
In this example, the binning function $g$ assigns two separate $W$-intervals (top panel) to the same bin (bin 2) as seen is the bottom panel. The shaded rectangles have widths $P_i$, heights $p_i$ and areas
$P_{i1}=P_i p_1$. In the upper panel, the conditional probabilities $p_i$ are monotonically increasing because they are averages of the monotonically increasing curve $p_1(w)$.
}
\label{binningFig}
\end{figure}

\begin{theorem}
\label{ContiguousBinningTheorem}
{\bf (Contiguous Binning Theorem):} If $W$ has a uniform distribution and the conditional probability distribution $P(W|Y\hbox{=}1)$ is monotonically increasing,
then all points $(H_*,I_*)$ on the Pareto frontier correspond to binning $W$ into contiguous intervals, \ie, if
\beq{IstarDefEq}
I(H_*) \equiv\sup_{\{g: H[g(W)]\le H_*\}} I[g(W),Y],
\eeq
then there exists a set of bin boundaries $b_1<...<b_{n-1}$
such that the binned variable $Z\equiv B(W,\b)\in\{1,...,M\}$ satisfies $H(Z)=H_*$ and $I(Z,Y)=I_*$.
\end{theorem}
\begin{proof}
We prove this by contradiction:
we will assume that there is a point 
$(H_*,I_*)$
on the Pareto frontier to which we can come arbitrarily close with
$\left(H(Z),I(Z,Y)\right)$ for $Z\equiv g(X)$ for a 
compression function $g: \mathbb{R}\mapsto \{1,...,M\}$
that is not a contiguous binning function, and obtain a contradiction by using $g$ to construct another
compression function $g'(W)$ lying above the Pareto frontier, 
with $H[g'(W)]=H_*$ and $I[g'(W),Y])>I_*$.
The joint probability distribution $P_{ij}$ for the $Z$ and $Y$ is given by the Lebesgue integral
\beq{LebesgueEq}
P_{ij}\equiv P(Z\hbox{=}i,Y\hbox{=}j) = \int f_j d\mu_i,
\eeq
where $f_j(w)$ is the  joint probability distribution for $W$ and $Y$ introduced earlier and 
$\mu_j$ is the set
$\mu\equiv\{w\in [0,1]: g(w)=i\}$, \ie, the set of $w$-values that are grouped together into the
$i^{\rm th}$ large bin.
We define the marginal and conditional probabilities
\beq{MargCondProbDef}
P_i\equiv P(Z\hbox{=}i)=P_{i1}+P_{i2}, \quad	p_i\equiv P(Y\hbox{=}1|Z\hbox{=}i)={P_{i1}\over P_i}.
\eeq
\Fig{binningFig} illustrates the case where the binning function $g$ corresponds to $M=4$ large bins, the second of which consists of two non-contiguous regions that are grouped together; the shaded rectangles in the bottom panel have width $P_i$, height $p_i$ and area $P_{ij}=P_i p_i$.


According to Theorem~\ref{informationTheorem} in the Appendix, we obtain the contradiction required to complete our proof
(an alternative compression $Z'\equiv g'(W)$ above the Pareto frontier with $H(Z')=H_*$
and $I(Z',Y)>I_*$) if there are two different conditional probabilities $p_k\ne p_l$, 
and we can change $g$ into $g'$ so that the joint distribution $P'_{ij}$ of $Z'$ and $Y$ changes in the following way:
\begin{enumerate}
\item Only $P_{kj}$ and $P_{lj}$ change,
\item both marginal distributions remain the same,
\item the new conditional probabilities $p'_k$ and $p'_l$ are further apart.
\end{enumerate}
\Fig{binningFig} shows how this can be accomplished for non-contiguous binning:
let $k$ be a bin with non-contiguous support set $\mu_k$ (bin 2 in the illustrated example), let $l$ be a bin whose support $\mu_l$ (bin 4 in the example) contains a positive measure subset 
$\mu_l^{\rm mid}\subset\mu_l$ within two parts $\mu_k^{\rm left}$ and $\mu_k^{\rm right}$ of $\mu_k$, 
and define a new binning function $g'(w)$ that differs from $g(w)$ only by swapping a set $\mu^\epsilon\subset \mu_l^{\rm mid}$ against a subset of either $\mu_k^{\rm left}$ or $\mu_k^{\rm right}$ of measure $\epsilon$ (in the illustrated example, the binning function change implementing this subset is shown with dotted lines).
This swap leaves the total measure of both bins (and hence the marginal distribution $P_i$) unchanged, and also leaves $P(Y\hbox{=}1)$ unchanged. 
If $p_k<p_l$, we perform this swap between $\mu_l^{\rm mid}$ an $\mu_k^{\rm right}$ (as in the figure), 
and if $p_k>p_l$, we instead perform this swap between $\mu_l^{\rm mid}$ an $\mu_k^{\rm left}$, in both cases guaranteeing that $p_l$ and $p_k$ move further apart (since $p(w)$ is monotonically increasing).
This completes our proof by contradiction except for the case where $p_k=p_l$; in this case, we swap to entirely eliminate the discontiguity, and repeat our swapping procedure between other bins until we increase the entropy (again obtaining a contradiction) or end up with a fully contiguous binning (if needed, $g(w)'$ can be changed to eliminate any measure-zero subsets that ruin contiguity, since they leave the Lebesgue integral in \eq{LebesgueEq} unchanged.)
\end{proof}

\subsection{Mapping the frontier}

Theorem~\ref{ContiguousBinningTheorem} implies that we can in practice find the Pareto frontier for any random variable $X$ by searching the space of contiguous binnings of $W=w(X)$ after uniformization, binning and sorting. 
In practice, we can first try the 2-bin case by scanning the bin boundary $0<b_1<1$, 
then trying the 3-bin case by trying bin boundaries $0<b_1<b_2<1$, then trying the 4-bin case, {\etc}, as illustrated in \fig{paretoAnalyticFig}.
Each of these cases corresponds to a standard multi-objective optimization problem aiming to maximize the two objectives $I(Z,Y)$ and $H(Z)$. We perform this optimization numerically with the AWS algorithm of \cite{kim2005adaptive} as described in the next section. 

Although the uniformization, binning and sorting procedures are helpful in practice as well as for for simplifying proofs, they are not necessary in practice.
Since what we really care about is grouping into integrals containing similar conditional probabilities $p_1(w)$, not similar $w$-values,
it is easy to see that binning horizontally after sorting is equivalent to binning vertically before sorting. 
In other words, we can eliminate the binning and sorting steps if we replace ``horizontal" binning
$g(W) = B(W,\b)$
by ``vertical" binning 
\beq{verticalBinningEq}
g(W) = B[p_1(W),\b],
\eeq
where $p_1$ denotes the conditional probability as before.


\section{Results}
\label{ResultsSec}

The purpose of this section is to examine how our method for Pareto-frontier mapping works in practice on various datasets, both to compare its performance with prior work and to gain insight into the shape and structure of the Pareto frontiers for well-known datasets such as the CIFAR-10 image database \cite{krizhevsky2014cifar}, the MNIST database of hand-written digits \cite{lecun2010mnist} and the Fashion-MNIST database of garment images \cite{xiao2017fashion}.
Before doing this, however, let us build intuition for how our method works by testing on a much simpler toy model that is analytically solvable, where the accuracy of all approximations can be exactly determined.

\subsection{Analytic warmup example}

Let the random variables $X=(x_1,x_2)\in[0,1]^2$ and 
$Y\in\{1,2\}$ be defined by the bivariate probability distribution
\beq{xProbDistEq}
f(X,Y) =  
\left\{
\begin{tabular}{ll}
$2 x_1 x_2$				&if $Y=1$,\\
$2(1-x_1)(1-x_2)$			&if $Y=2$,
\end{tabular}
\right.
\eeq
which corresponds to $x_1$ and $x_2$ being two independent and identically distributed random variables with triangle distribution $f(x_i)=x_i$ if $Y=1$, but flipped $x_i\mapsto 1-x_i$ if $Y=2$.
This gives $H(Y)=1$ bit and mutual information
\beq{AnalyticIeq}
I(X,Y)=1-{\pi^2-4\over 16\ln 2}\approx 0.4707\>\hbox{bits}.
\eeq
The compressed random variable $W=w(X)\in\mathbb{R}$ defined by \eq{wDefEq} is thus
\beq{analyticgEq}
W=
P(Y\hbox{=}1|X) 
= {x_1 x_2\over x_1 x_2+(1-x_1)(1-x_2)}.
\eeq
After defining $Z\equiv B(W,\b)$ for a vector $\b$ of bin boundaries, 
a straightforward calculation shows that the joint probability distribution of $Y$ and the binned variable $Z$ is given by
\beq{AnalyticPeq}
P_{ij}\equiv P(Z\hbox{=}i,Y\hbox{=}j) = 
F_j(b_{i+1})-F_j(b_i),
\eeq
where the cumulative distribution function
$F_j(w)\equiv P(W\hbox{$<$}w,Y\hbox{=}j)$
is given by
\beqa{analyticFwEq}
F_1(w)&=&{w^2\left[(2 w - 1) (5 - 4 w) + 
     2 (1 - w^2) \log(w^{-1} - 1)\right]\over 2 (2 w-1)^4},\nonumber\\
     F_2(w)&=&{1\over 2} - F_1(1-w).
\eeqa   
Computing $I(W,Y)$ using this probability distribution recovers exactly the same mutual information $I\approx 0.4707\>$bits as in \eq{AnalyticIeq}, as we proved in Theorem~\ref{Wtheorem}.

\subsection{The Pareto frontier}

Given any binning vector $\b$, we can plot a corresponding point 
$(H[Z],I[Z,Y])$ in \fig{paretoAnalyticFig} by computing 
$I(Z,Y)=H(Z)+H(Y)-H(Z,Y)$,\\
$H(Z,Y)=-\sum P_{ij}\log P_{ij}$, {\etc},
where $P_{ij}$ is given by \eq{AnalyticPeq}.

The figure shows 6,000 random binnings each for $M=3,...,8$ bins; as we have proven, the upper envelope of points corresponding to all possible (contiguos) binnings defines the Pareto frontier.
The Pareto frontier begins with the black dot at $(0,0)$ (the lower right corner), since
$M=1$ bin obviously destroys all information. The $M=2$ bin case corresponds to a 1-dimensional closed curve parametrized by the single parameter $b_1$ that specifies the boundary between the two bins: 
it runs from $(0,0)$ when $b_1=1$, moves to the left until $H(Z)=1$ when $b_1=0.5$, and returns to $(0,0)$ when $b_1=1$.  The $b_1<0.5$ and $b_1>0.5$ branches are indistinguishable in \fig{paretoAnalyticFig} because of the symmetry of our warmup problem, but in generic cases, a closed loop can be seen where only the upper part defines the Pareto frontier.

More generally, we see that the set of all binnings into $M>2$ bins maps the vector $\b$ of $M-1$ bin boundaries into a contiguous region in \fig{paretoAnalyticFig}. The inferior white region region below  can also be reached if we use non-contiguous binnings.

The Pareto Frontier is seen to resemble the  top of a circus tent, with convex segments separated by ``corners" where the derivative vanishes, corresponding to a change in the number of bins. We can understand the origin of these corners by considering what happens when adding a new bin of infinitesimal size $\epsilon$. 
As long as $p_i(w)$ is continuous, this changes all probabilites $P_{ij}$ by amounts 
$\delta P_{ij}=\mathcal{O}(\epsilon)$, and the probabilities corresponding to the new bin
(which used to vanish) will now be ${O}(\epsilon)$.
The function $\epsilon\log\epsilon$ has infinite derivative at $\epsilon=0$, blowing up as 
$\mathcal{O}(\log\epsilon)$, which implies that 
the entropy increase $\delta H(Z)=\mathcal{O}(-\log\epsilon)$.
In contrast, a straightforward calculation shows that all $\log\epsilon$-terms cancel when computing the mutual information, which changes only by 
$\delta I(Z,Y)=\mathcal{O}(\epsilon)$. As we birth a new bin and move leftward from one of the black dots in \fig{paretoAnalyticFig}, the initial slope of the Pareto frontier is thus
\beq{ParetoSlopeEq}
\lim_{\epsilon\to 0}\>{\delta I(Z,Y)\over \delta H(Z)}=0.
\eeq
In other words, the Pareto frontier starts out {\it horizontally} to the left of each of its corners in \fig{paretoAnalyticFig}. Indeed, the corners are ``soft" in the sense that the derivative of the Pareto Frontier is continuous and vanishes at the corners: for a given number of bins, $I(X,Z)$ by definition takes its global maximum at the corresponding corner, so the derivative $\partial I(Z,Y)/\partial H(Z)$ vanishes also as we approach the corner from the right.\footnote{The first corner (the transition from 2 to 3 bins) can nonetheless look fairly sharp because the 2-bin curve turns around rather abruptly, and the right derivative does not vanish in the limit where a symmetry causes the upper and lower parts of the 2-bin loop to coincide.}


Our theorems imply that in the $M\to\infty$ limit of infinitely many bins, 
successive corners become gradually less pronounced (with ever smaller derivative discontinuities), because the left asymptote of the Pareto frontier simply approaches the horizontal line $I_*=I(Y,Z)$.

\subsubsection{Approximating $w(X)$}

For our toy example, we knew the conditional probability distribution $P(Y|X)$ and could therefore compute $W=w(X)=P(Y\hbox{=}1|X)$ exactly. For practical examples where this is not the case, we can instead train a neural network to implement a function $\what(X)$ that approximates 
$P(Y\hbox{=}1|X)$. For our toy example, we train a fully connected feedforward neural network to predict $Y$ from $X$ using cross-entropy loss;  it has 2 hidden layers, each with 256 neurons with ReLU activation, and a final linear layer with softmax activation, whose first neuron defines $\what(X)$.
A illustrated in \fig{contourFig}, the network prediction for the conditional probability 
$\what(X)\equiv P(Y\hbox{=}1)$ is fairly accurate, but slightly over-confident, 
tending to err on the side of predicting more extreme probabilities (further from $0.5$). The average KL-divergence between the predicted and actual conditional probability distribution $P(Y|X)$ is  about $0.004$, which causes negligible loss of information about $Y$.

\begin{figure}[h]
\begin{center}
\includegraphics[width=\columnwidth]{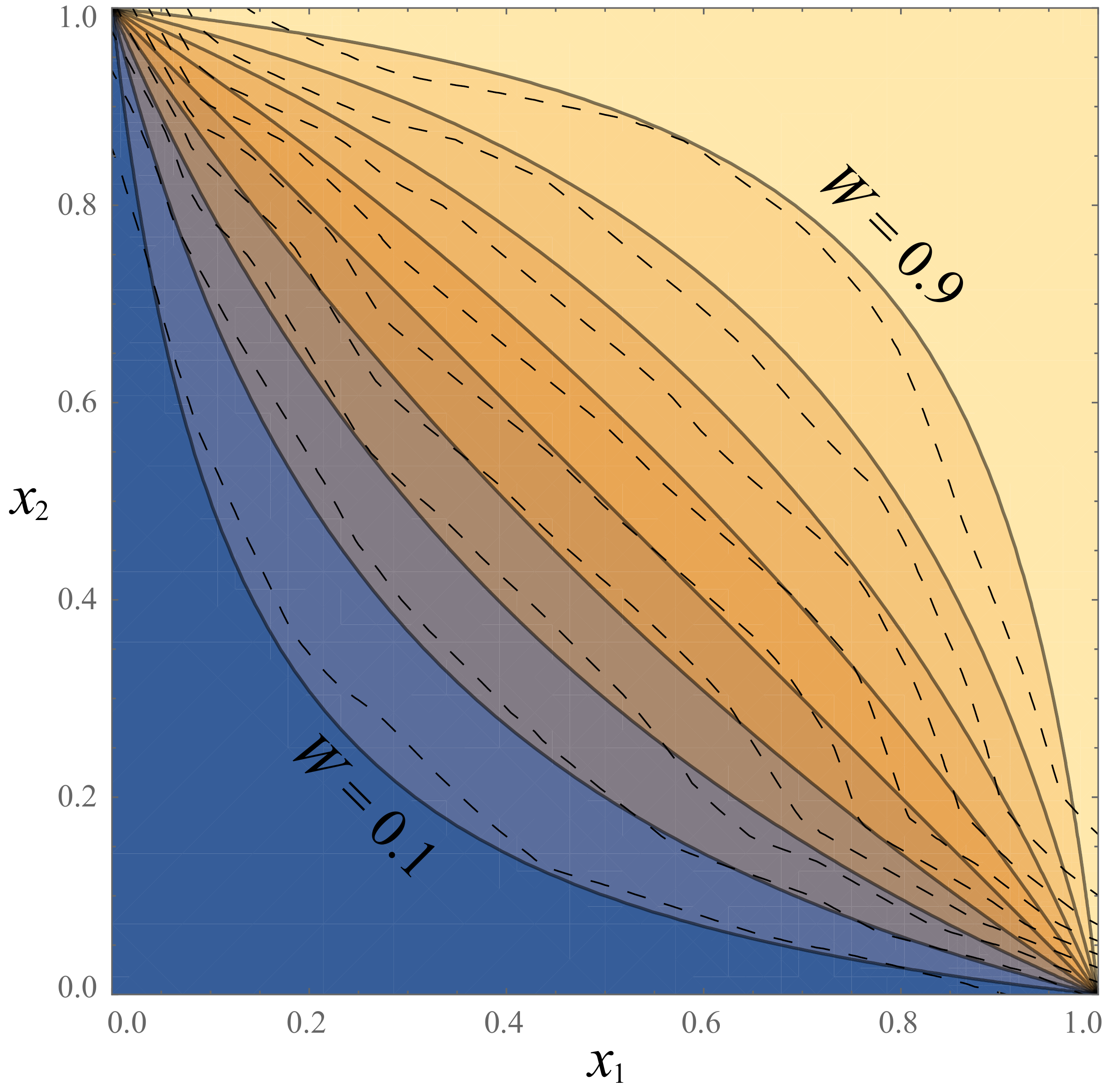}
\end{center}
\vskip-5mm
\caption{Contour plot of the function $W(x_1,x_2)$ computed both exactly using 
equation~\ref{analyticgEq} (solid curves) and approximately using a neural network (dashed curves).
}
\label{contourFig}
\end{figure}

\subsubsection{Approximating $f_1(W)$}

For practical examples where the conditional joint probability distribution $P(W,Y)$ cannot be computed analytically, we need to estimate it from the observed distribution of $W$-values output by the neural network. 
For our examples, we do this by fitting 
each probability distribution by a beta-distribution times the exponential of a polynomial of degree $d$: 
\beq{parametrizedFuncDefEq}
f(w,\a)\equiv 
\exp\left[\sum_{k=0}^d a_k x^k\right]x^{a_{d+1}}(1-x)^{a_{d+2}}, 
\eeq
where the coefficient $a_0$ is fixed by the normalization requirement $\int_0^1 f(w,\a)dw=1$.
We use this simple parametrization because it 
can fit any smooth distribution arbitrarily well for sufficiently large $d$, and 
provides accurate fits for the probability distributions in our examples using quite modest $d$; for example, $d=3$ gives 
$d_{KL}[ f_1(w),f(w,\a)]\approx 0.002$
for
\beqa{pFitLikelihoodEq}
\a&\equiv&\underset{\a'}{\rm argmin}\>\>d_{\rm KL}[f_1(w),f(w,\a')]\\
    &=&(-1.010, 2.319, -5.579, 4.887, 0.308, -0.307),\nonumber
\eeqa
which causes rather negligible loss of information about $Y$.
For our examples below where we do not  know the exact distribution $f_1(w)$ and merely have samples $W_i$ drawn from it, one for each element of the data set, we
instead perform the fitting by the standard technique of minimizing the cross entropy loss, \ie, 
\beq{pFitLikelihoodEq2}
\a\equiv \underset{\a'}{\rm argmin} \>\>-\sum_{k=1}^{n}\log f(W_k,\a').
\eeq
Table~\ref{ProbFitTable} lists the fitting coefficients used, and \fig{ProbFitsFig} illustrates the fitting accuracy.

\begin{table*}[]
{\footnotesize
\begin{tabular}{|l|l|r|r|r|r|r|r|r|}
\hline
Experiment   & Y &$a_0$& $a_1$ & $a_2$ & $a_3$  &$a_4$ &$a_5$&$a_6$\\
\hline                            
Analytic		&1	&0.0668	&-4.7685	&16.8993	&-25.0849		&13.758	&0.5797	&-0.2700\\
			&2	&0.4841	&-5.0106	&5.7863	&-1.5697		&-1.7180	&-0.3313	&-0.0030\\
\hline
Fashion-MNIST	&Pullover	&0.2878	&-12.9596 &44.9217	&-68.0105 &37.3126	&0.3547	&-0.2838\\
			&Shirt	&1.0821	&-23.8350	&81.6655	&-112.2720	&53.9602	&-0.4068	&0.4552\\
\hline
CIFAR-10		&Cat		&0.9230	&0.2165	&0.0859	&6.0013	&-1.0037	&0.8499\\
			&		&		&0.6795	&0.0511	&0.6838	&-1.0138	&0.9061\\ 
  			&Dog	&0.8970	&0.2132	&0.0806	&6.0013	&-1.0039	&0.8500\\
			&		&		&0.7872	&0.0144	&0.7974	&-0.9440	&0.7237\\			
\hline
MNIST		&One	&3.1188	&-65.224	&231.4	&-320.054	&150.779	&1.1226	&-0.6856\\
			&Seven	&-1.0325	&-47.5411	&189.895	&-269.28	&127.363	&-0.8219	&0.1284\\
\hline
\end{tabular}
\caption{Fits to the conditional probability distributions $P(W|Y)$ for our experiments, 
in terms of the parameters $a_i$ defined by \protect\eq{parametrizedFuncDefEq}.
\label{ProbFitTable}
}
}
\end{table*}

\subsection{MNIST, Fashion-MNIST and CIFAR-10}

The MNIST database consists of 28x28 pixel greyscale images of handwritten digits: 
60,000 training images and 10,000 testing images \cite{lecun2010mnist}.
We use the digits 1 and 7, since they are the two that are most frequently confused, relabeled as $Y=1$ (ones) and $Y=2$ (sevens). To increase difficulty, we inject 30\% of pixel noise, i.e., randomly flip each pixel with 30\% probability (see examples in \fig{classesFig}). For easy comparison with the other cases, we use the same number of samples for each class. 

The Fashion-MNIST database has the exact same format (60,000 + 10,000 28x28 pixel greyscale images), depicting not digits but 10 classes of clothing \cite{xiao2017fashion}. Here we again use the two most easily confused classes: pullovers ($Y=1$) and shirts ($Y=2$);  see \fig{classesFig} for examples.

The architecture of the neural network classifier we train on the above two datasets is adapted from here\footnote{We use the neural network architecture from \href{https://github.com/pytorch/examples/blob/master/mnist/main.py}{github.com/pytorch/examples/blob/master/mnist/main.py}; the only difference in architecture is that our output number of neurons is 2 rather than 10.}: two convolutional layers (kernel size 5, stride 1, ReLU activation) with 20 and 50 features, respectively, each of which is followed by max-pooling with kernel size 2. This is followed by a fully connected layer with 500 ReLU neurons and finally a softmax layer that produces the predicted probabilities for the two classes.  After training, we apply the trained model to the test set to obtain $W_i=P(Y|X_i)$ for each dataset.

CIFAR-10 \cite{cifar} is one of the most widely used datasets for machine learning research, and contains 60,000 $32\times 32$ color images in 10 different classes.
We use only the cat ($Y=1$) and dog ($Y=2$) classes, which are the two that are empirically hardest to discriminate;  see \fig{classesFig} for examples.
We use a ResNet18 architecture\footnote{The architecture is adapted from \href{https://github.com/kuangliu/pytorch-cifar}{github.com/kuangliu/pytorch-cifar}, for which we use its ResNet18 model; the only difference in architecture is that we use 2 rather than 10 output neurons.} \cite{he2016deep}. We train with a learning rate of 0.01 for the first 150 epochs, 0.001 for the next 100, and 0.0001 for the final 100 epochs; we keep all other settings the same as in the original repository.

\begin{figure}[h]
\begin{center}
\includegraphics[width=\columnwidth]{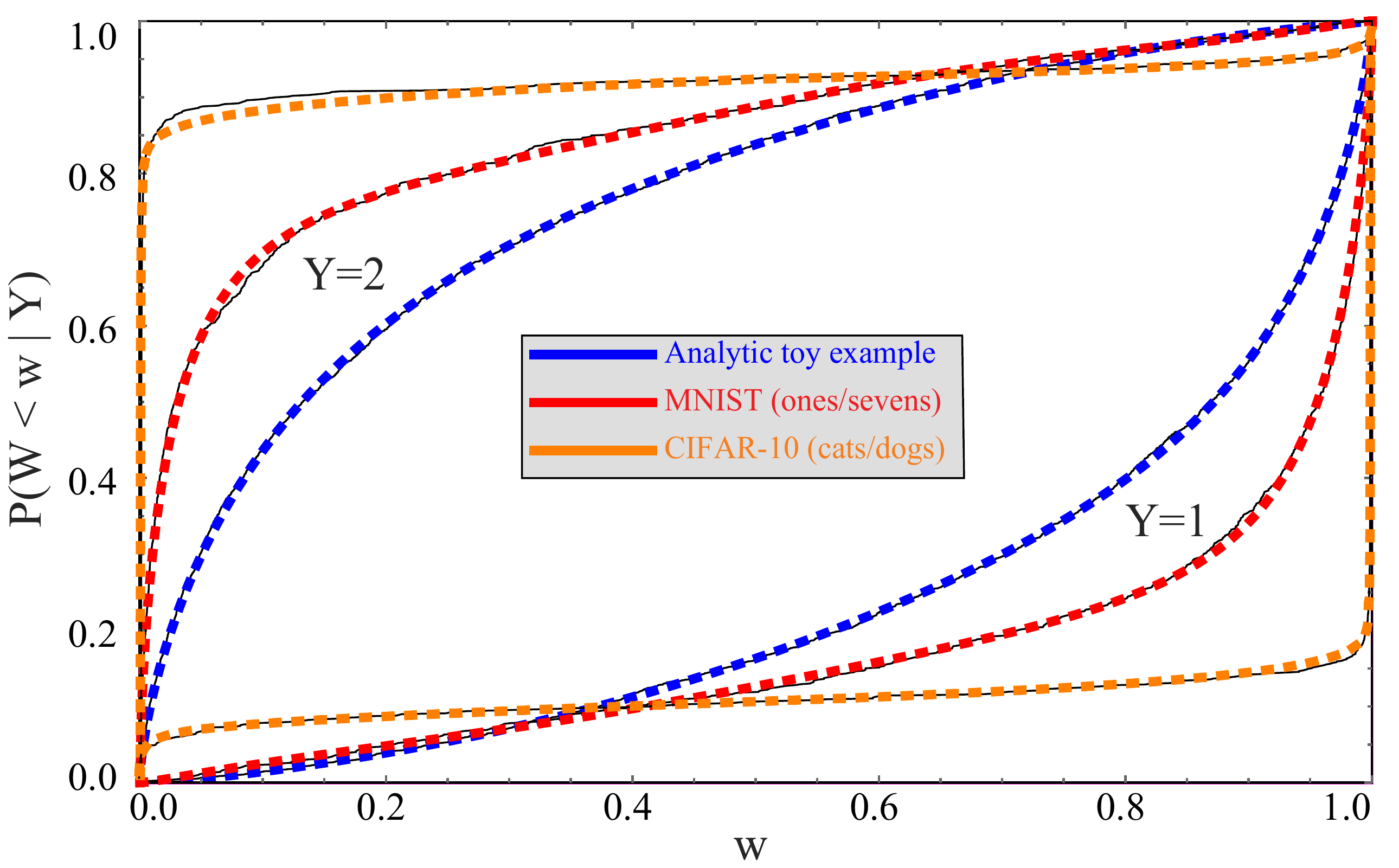}
\end{center}
\vskip-5mm
\caption{Cumulative distributions $F_i(w)\equiv P(W\hbox{$<$}w|Y\hbox{$=$}i)$ are shown for the analytic (blue/dark grey), Fashion-MNIST (red/grey) and CIFAR-10 (orange/light grey) examples. Solid curves show the observed cumulative histograms of $W$ from the neural network, and dashed curves show the fits defined by \eq{parametrizedFuncDefEq} and
Table~\ref{ProbFitTable}.
}
\label{ProbFitsFig}
\end{figure}

\Fig{ProbFitsFig} shows observed cumulative distribution functions $F_i(w)$ (solid curves) for the $W_i=P(Y=1|X_i)$ generated by the neural network classifiers, together with our above-mentioned analytic fits (dashed curves).\footnote{In the case of CIFAR-10, the observed distribution $f(w)$ was so extremely peaked near the endpoints that we replaced \eq{parametrizedFuncDefEq} by the more accurate fit 
\beqa{CIFARfitEq1}
f(w)&\equiv		& F'(w),\\
F(w)&\equiv		&\left\{
\begin{tabular}{lc}
$\a^A_0 F_*[w, \a^A]$			&if $w<1/2$,\\
$1 - (1 - \a^A_0) F_*[1 - w, \a^B]]$	&otherwise,
\end{tabular}
\right.\\
F_*(x)&\equiv&G\left[{(2 x)^{a_1}\over 2}\right],\\
G(x)&\equiv&\left[\left({x\over a_2}\right)^{a_3 a_4} + (a_5 + a_6 x)^{a_4}\right]^{1/a_4},\\
a_6&\equiv&2\left[(1 - (2 a_2)^{-a_3 a_4})^{1/a_4} - a_5\right],
\eeqa
where the parameters vectors $\a^A$ and $\a^B$ are given in Table~\ref{ProbFitTable} for both cats and dogs. For the cat case, this fit gives not $f(w)$ but $f(1-w)$.
Note that $F_*(0) = 0$, $F_*(1/2)=1$.
}
\Fig{CondProbFig} shows the corresponding conditional probability curves $P(Y=1|W)$ after remapping $W$ to have a uniform distribution as described above.  \Fig{ProbFitsFig} shows that the original $W$-distributions are strongly peaked around $W\approx 0$ and $W\approx 1$, so this remapping stretches the $W$-axis so as to shift probability toward more central values.
  
 \begin{figure}[h!]
\begin{center}
\includegraphics[width=\columnwidth]{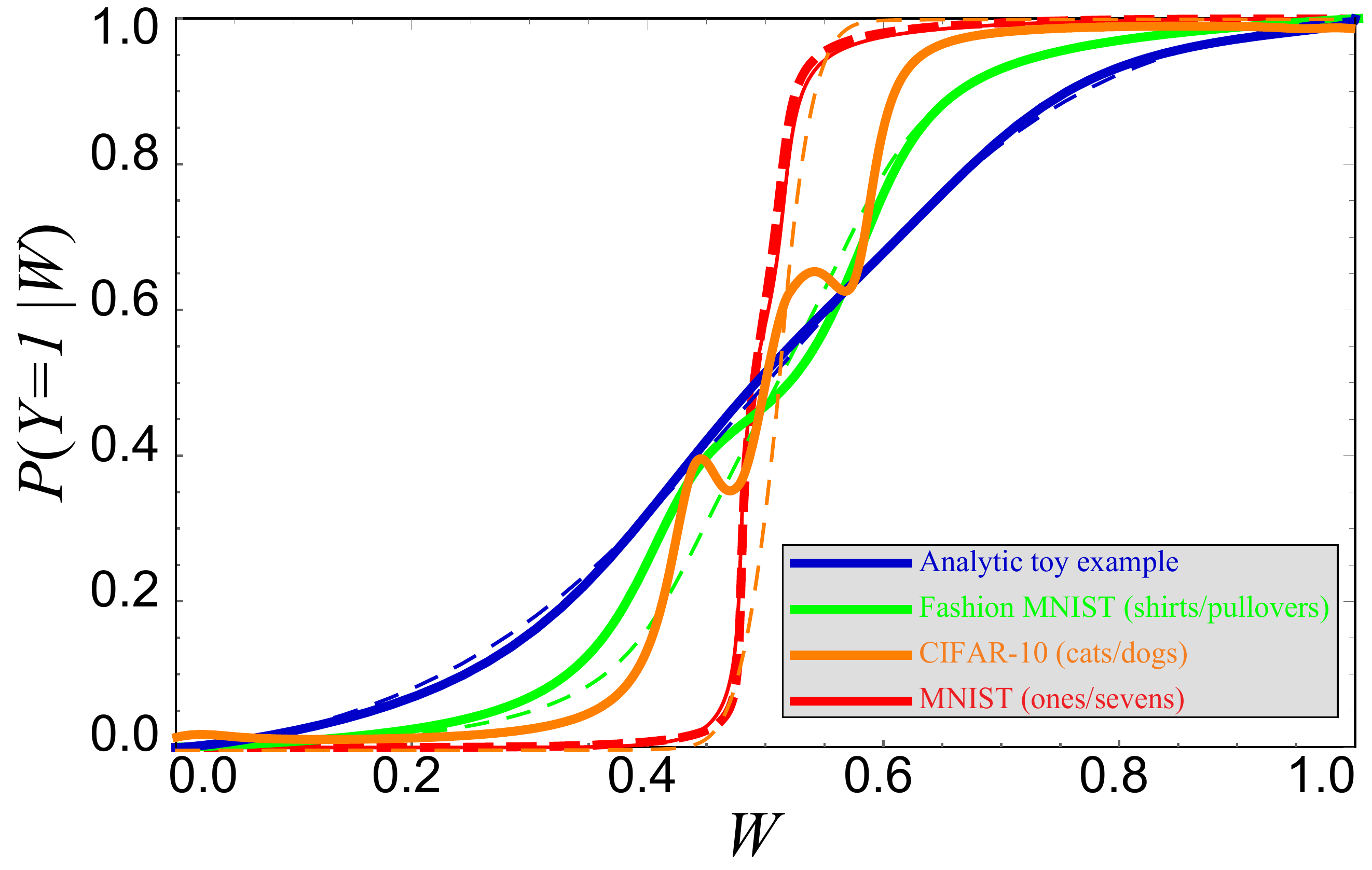}
\end{center}
\vskip-4mm
\caption{The solid curves show the actual conditional probability $P(Y\hbox{=}1|W)$ for CIFAR-10 (where the labels Y=1 and 2 correspond to ``cat" and ``dog") and MNIST with 20\% label noise (where the labels Y=1 and 2 correspond to ``1" and ``7") , respectively.
The color-matched dashed curves show the conditional probabilities predicted by the neural network; the reason that they are not diagonal lines  $P(Y\hbox{=}1|W)=W$ is that $W$ has been reparametrized to have a uniform distribution.
If the neural network classifiers were optimal, then solid and dashed curves would coincide.}
\label{CondProbFig}
\end{figure}

\begin{figure}[h]
\begin{center}
\includegraphics[width=\columnwidth]{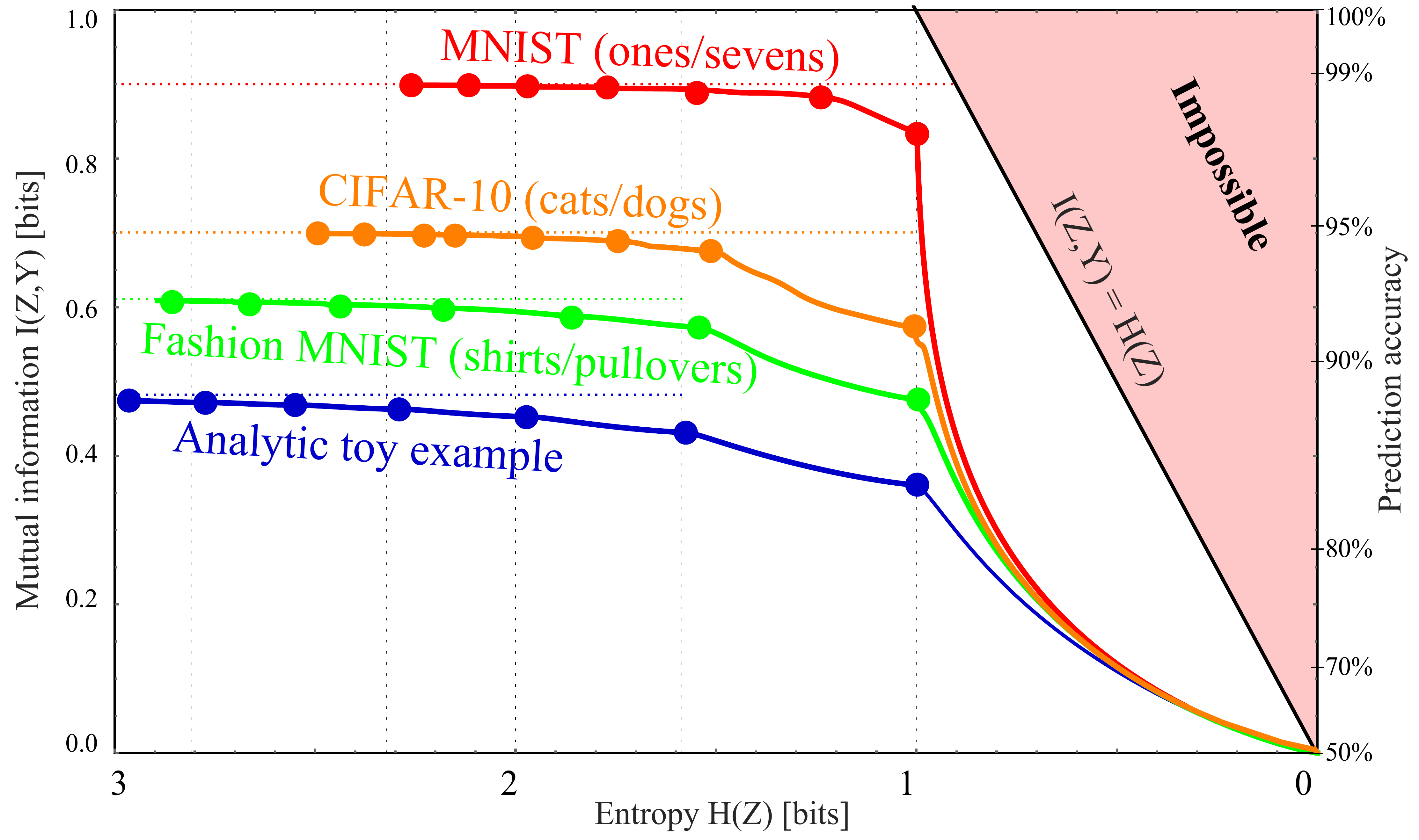}
\end{center}
\vskip-5mm
\caption{The Pareto frontier for compressed versions $Z=g(X)$ of our four datasets $X$, showing the maximum attainable class information $I(Z,Y)$ for a given entropy $H(Z)$.
The ``corners" (dots) correspond to the maximum $I(Z,Y)$ attainable when binning the likelihood $W$ into a given number of bins (2, 3, ...,8 from right to left). The horizontal dotted lines show the maximum available information $I(X,Y)$ for each case, reflecting that there is simply less to learn in some examples than in others.
}
\label{paretoFig}
\end{figure}

The final result of our calculations is shown in \Fig{paretoFig}: the Pareto frontiers for our four datasets, computed using our method. 

\subsection{Interpretation of our results}

To build intuition for our results, let us consider our CIFAR-10 example of images $X$ depicting cats ($Y\hbox{=}1$) and dogs  ($Y\hbox{=}2$) as in \fig{classesFig} and ask what aspects $Z=g(X)$ of an image $X$ capture the most information about the species $Y$.
Above, we estimated that 
$I(X,Y)\approx 0.69692$ 
bits, so what $Z$ captures the largest fraction of this information for a fixed entropy?
Given a good neural network classifier, a natural guess might be the single bit $Z$ containing its best guess, say ``it's probably a cat".
This corresponds to defining $Z=1$ if $P(Y\hbox{=}1|X)>0.5$, $Z=2$ otherwise, and gives the joint distribution 
$P_{ij}\equiv P(Y\hbox{=}i,Z\hbox{=}j$)
$$
\P=\left(
\begin{tabular}{cc}
0.454555		&0.045445\\
0.042725		&0.457275
\end{tabular}
\right)
$$
corresponding to $(Z,Y)\approx 0.56971$ bits.
But our results show that we can improve things in two separate ways.

First of all, if we only want to store one bit $Z$, then we can do better, corresponding to the first ``corner" in \Fig{paretoFig}:
moving the likelihood cutoff from $0.5$ to $0.51$, \ie, redefining $Z=1$ if $P(Y|X)>0.51$,
increases the mutual information to $I(Z,Y)\approx 0.56974$ bits.

More importantly, we are still falling far short of the $0.69692$ bits of information we had without data compression, capturing only 88\% of the available species information.
Our Theorem~\ref{Wtheorem} showed that we can retain {\it all} this information if we instead define $Z$ as the cat probability itself:
$Z\equiv W\equiv P(Y|X)$. For example, a given image might be compressed not into {\it ``It's probably a cat"} 
but into {\it ``I'm 94.2477796\% sure it's a cat"}.
However, it is clearly impractical to report the infinitely many decimals required to retain all the species information, which would make $H(Z)$ infinite.
Our results can be loosely speaking interpreted as the optimal way to round $Z$, so that the information  $H(Z)$ required to store it becomes finite.
We found that simply rounding to a fixed number of decimals is suboptimal; for example, if we pick 2 decimals and say
{\it ``I'm 94.25\% sure it's a cat"}, then we have effectively binned the probability $W$ into 10,000 bins of equal size, even though we can often do much better with bins of unequal size, as illustrated in the bottom panel of 
\fig{paretoAnalyticFig}. Moreover, when the probability $W$ is approximated by a neural network, we found that what should be optimally binned is not $W$ but the conditional probability $P(Y\hbox{=}1|W)$ illustrated in \fig{CondProbFig}
(``vertical binning").

It is convenient to interpret our Pareto-optimal data compression $X\mapsto Z$ as {\it clustering}, \ie, as a method of grouping our images or other data $X_i$ into clusters based on what information they contain about $Y$.
For example, \fig{classesFig} illustrates CIFAR-10 images clustered by their degree of ``cattiness"
into 5 groups $Z=1,...,5$ that might be nicknamed ``1.9\% cat", ``11.8\% cat", ``31.4\% cat", ``68.7\% cat" and ``96.7\% cat".  
This gives the joint distribution $P_{ij}\equiv P(Y\hbox{=}i,Z\hbox{=}j$) where
$$
\P=\left(
\begin{tabular}{ccccc}
0.350685	&0.053337	&0.054679	&0.034542	&0.006756\\
0.007794	 &0.006618	&0.032516	&0.069236	&0.383836
\end{tabular}
\right)
$$
and gives $I(Z,Y)\approx 0.6882$, thus increasing the fraction of species information retained from 
82\% to 99\%.

This is a striking result: we can group the images into merely five groups and discard all information about all images except which group they are in, yet retain 99\% of the information we cared about.
Such grouping may be helpful in many contexts. For example, given a large sample of labeled medical images of potential tumors, they can be used to define say five optimal clusters, after which future images can be classified into five degrees of cancer risk that collectively retain virtually all the malignancy information in the original images.

Given that the Pareto Frontier is continuous and corresponds to an infinite family of possible clusterings, which one is most useful in practice? Just as in more general multi-objective optimization problems, the most interesting points on the frontier are arguably its ``corners", indicated by dots in \fig{paretoFig}, where we do notably well on both criteria. 
This point was also made in the important paper \cite{strouse2019information} in the context of the DIB-frontier discussed below.
We see that the parts of the frontier between corners tend to be convex and thus rather unappealing, since any weighted average of $-H(Z)$ and $I(Z,Y)$ will be maximized at a corner.
Our results show that these corners can conveniently be computed without numerically tedious multiobjective optimization, by simply maximizing the mutual information $I(Z,Y)$ for $m=2, 3, 4, ...$ bins. The first corner, at $H(Z)=1$bit, corresponds to the learnability phase transition for DIB, \ie, the largest $\beta$ for which DIB is able to learn a non-trivial representation. In contrast to the IB learnability phase transition \citep{wu2019learnability,wu2019learnabilityEntropy} where $I(Z,Y)$ increases continuously from 0, here the $I(Y;Z)$ has a jump from 0 to a positive value, due to the non-concave nature of the Pareto frontier.

Moreover, all the examples in \fig{paretoFig} are seen to get quite close to the $m\to\infty$ asymptote 
$I(Z,Y)\to I(X,Y)$ for $m\simgt 5$, so the most interesting points on the Pareto frontier are simply the first handful of corners. For these examples, we also see that the greater the mutual information is, the fewer bins are needed to capture most of it. 

An alternative way if interpreting the Pareto plane in \fig{paretoFig} 
is as a traveoff between two evils: 
\beqa{InfoBloatLossEq}
\hbox{\bf Information bloat:~} H(Z|Y)&\equiv&H(Z)-I(Z,Y)\ge 0,\nonumber\\
\hbox{\bf Information loss:~~~~~~~~$\>$} \Delta I&\equiv&I(X,Y)-I(Z,Y)\ge 0.\nonumber
\eeqa
What we are calling the {\it information bloat} has also been called ``causal waste" \cite{thompson2018causal}. It is simply the conditional entropy of $Z$ given $Y$, and represents the excess bits we need to store in order to retain the desired information about $Y$. 
Geometrically, it is the horizontal distance to the impossible region to the right in \fig{paretoFig}, and we see that for MNIST, it takes local minima at the corners for both 1 and 2 bins. The {\it information loss} is simply the information discarded by our lossy compression of $X$. Geometrically, it is the vertical distance to the impossible region at the top of \fig{paretoAnalyticFig} (and, in \fig{paretoFig}, it is the vertical distance to the corresponding dotted horizontal line). As we move from corner to corner adding more bins, we typically reduce the information loss at the cost of increased information bloat. 
For the examples in \fig{paretoFig}, we see that going beyond a handful of bins essentially just adds bloat without significantly reducing the information loss.




\subsection{Real-world issues}

We just discussed how lossy compression is a tradeoff between information bloat and information loss. Let us now elaborate on the latter, for the real-world situation where 
$W\equiv P(Y\hbox{=}1|X)$ is approximated by a neural network. 

If the neural network learns to become perfect, then the function $w$ that it implements will be such that $W\equiv w(X)$ satisfies 
$P(Y=1|W)=W$, which corresponds to the dashed curves in \fig{CondProbFig} being identical to the solid curves. Although we see that this is close to being the case for the analytic and MNIST examples, the neural networks are further from optimal for Fashion-MNIST and CIFAR-10.
The figure illustrates that the general trend is for these neural networks to overfit and therefore be overconfident, predicting probabilities that are too extreme.

This fact that $P(Y\hbox{=}1|W)\ne W$ probably indicates that our Fashion-MNIST and CIFAR-10 classifiers $W=w(X)$ destroy information about $X$, but it does not prove this, because if we had a perfect lossless classifier $W\equiv w(X)$ satisfying $P(Y\hbox{=}1|W)=W$, then we could define an overconfident lossless classifier by an invertible (and hence information-preserving) reparameterization such as $W'\equiv W^2$ that violates the condition $P(Y\hbox{=}1|W')=W'$.


So how much information does $X$ contain about $Y$? One way to lower-bound $I(X;Y)$ uses the classification accuracy:
if we have a classification problem where $P(Y\hbox{=}1)=1/2$ and compress $X$ into a single classification bit $Z$ (corresponding to a binning of $W$ into two bins), then we can write the 
joint probability distribution for $Y$ and the guessed class $Z$ as
$$
P=\left(
\begin{tabular}{cc}
${1\over 2}-\epsilon_1$&$\epsilon_1$\\
$\epsilon_2$	&${1\over 2}-\epsilon_2$
\end{tabular}
\right).
$$
For a fixed total error rate $\epsilon\equiv\epsilon_1+\epsilon_2$,
Fano's Inequality implies that the  mutual information takes a minimum
\beq{PredictionAccuracyEq}
I(Z,Y)=1 + \epsilon\log\epsilon + (1 - \epsilon)\log(1 - \epsilon)
\eeq
when $\epsilon_1=\epsilon_2=\epsilon/2$, 
so if we can train a classifier that gives an error rate $\epsilon$, 
then the right-hand-side of \eq{PredictionAccuracyEq} places a lower bound on the mutual information $I(X,Y)$.
The prediction accuracy $1-\epsilon$ is shown for reference on the right side of \fig{paretoFig}. Note that getting close to one bit of mutual information requires extremely high accuracy; for example, 99\% prediction accuracy corresponds to only 0.92 bits of mutual information.
 
We can obtain a stronger estimated lower bound on $I(X,Y)$ from the cross-entropy loss function $\Ell$ used to train our classifiers:
\beq{LossEq}
\expec{\Ell} = -\left<\log P(Y\hbox{=}Y_i|X\hbox{=}X_i)\right> = H(Y|X) + \dKL,
\eeq
where $\dKL\ge 0$ denotes the average KL-divergence between true and predicted conditional probability distributions, and $\expec{\cdot}$ denotes ensemble averaging over data points, which implies that
\beqa{IboundEq}
I(X,Y)&=&H(Y) - H(Y|X) = H(Y) - \expec{\Ell} -\dKL\nonumber\\
	&\ge& H(Y) - \expec{\Ell}.
\eeqa
If $P(Y\hbox{=}1|W)\ne W$ as we discussed above, then $\dKL$ and hence the loss can be further reduced be recalibrating $W$ as we have done, which increases the information bound from 
\eq{IboundEq} up to the the value computed directly from the observed joint distribution $P(W,Y)$.

Unfortunately, without knowing the true probability $p(Y|X)$, there is no rigorous and practically useful {\it upper} bound on the mutual information other than the trivial inequality $I(X,Y)<H(Y)=1$ bit, as the following simple counterexample shows: suppose our images $X$ are encrypted with some encryption algorithm that is extremely time-consuming to crack, rendering the images for all practical purposes indistinguishable from random noise. Then any reasonable neural network will produce a useless classifier giving $I(W,Y)\approx 0$ even though the true mutual information $I(X,Y)$ could be as large as one bit. In other words, we generally cannot know the true information loss caused by compressing $X\mapsto W$, so the best we can do in practice is to pick a corner reasonably close to the upper asymptote in \fig{paretoFig}.

\begin{figure}[h]
\begin{center}
\includegraphics[width=\columnwidth]{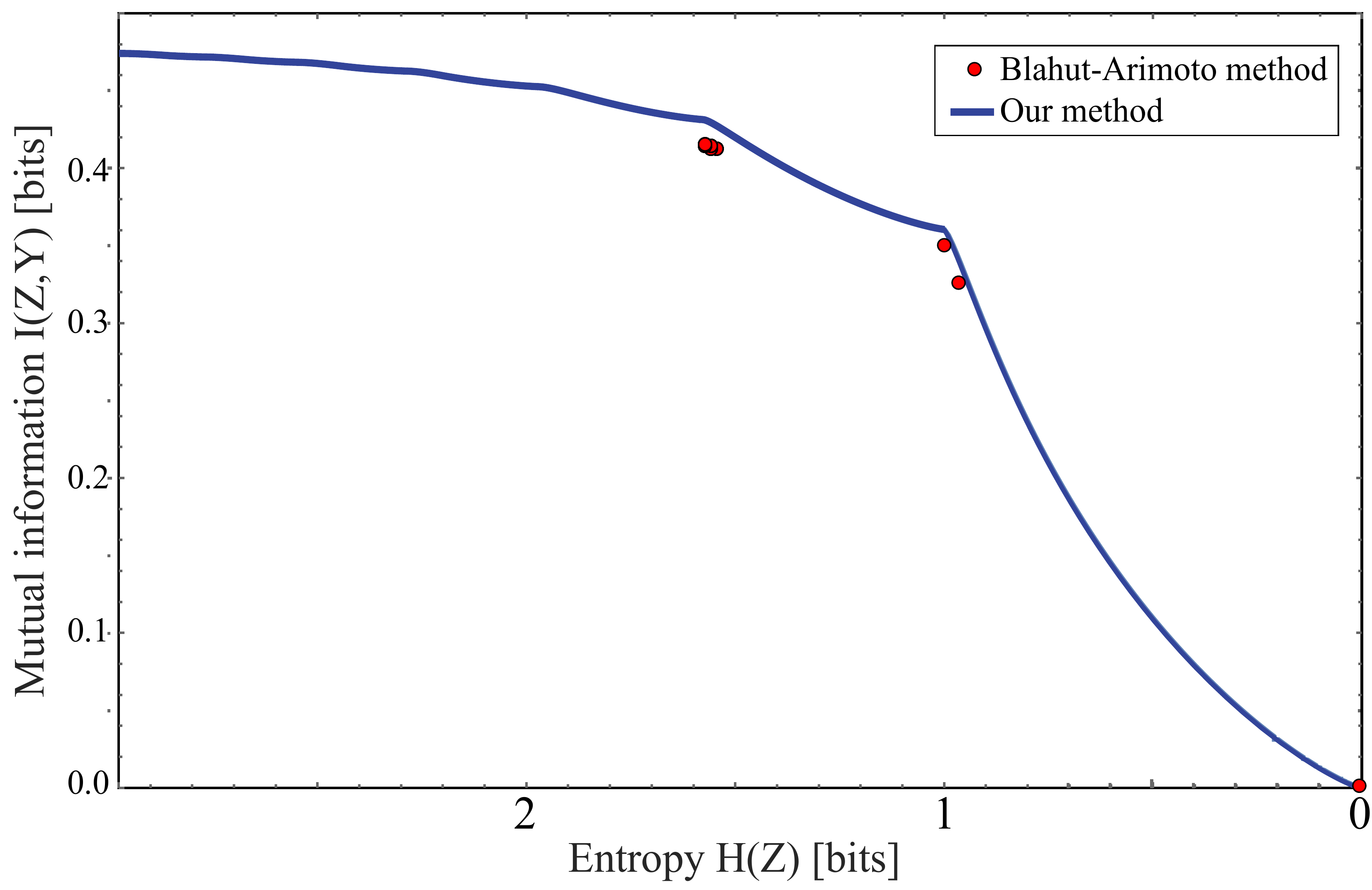}
\end{center}
\vskip-5mm
\caption{The Pareto frontier our analytic example is computed exactly with our method (solid curve) and approximately with the Blahut-Arimoto method (dots).}
\label{BlahutArimotoFig}
\end{figure}

\begin{table}[]
{\footnotesize
\begin{tabular}{|c|cc|}
\hline
$H(Z)$&\multicolumn{2}{c|}{$I(Z,Y)$}\\
        &BA-method&Our method\\
\hline                            
0.0000&0.0000&0.0000\\
0.9652&0.3260&0.3421\\
0.9998&0.3506&0.3622\\
1.5437&0.4126&0.4276\\
1.5581&0.4126&0.4298\\
1.5725&0.4141&0.4314\\
\hline
\end{tabular}
\caption{The approximate Pareto frontier points for our analytic example computed with the Blahut-Arimoto (BA) method compared with the points for those same six $H$-values computed with our exact method. 
\label{BlahutArimotoTable}
}
}
\end{table}

\subsection{Performance compared with Blahut-Arimoto method}

The most commonly used technique to date for finding the Pareto frontier is the Blahut-Arimoto (BA) method 
\cite{blahut1972computation,arimoto1972algorithm}
applied to the DIB objective of \eq{IBeq2} as described in  \cite{strouse2017deterministic}. 
\fig{BlahutArimotoFig} and Table~\ref{BlahutArimotoTable} shows the BA method implemented as in \cite{strouse2019information}, applied to our above-mentioned analytic toy example, after binning using 2,000 equispaced $W$-bins and $Z\in{1,...,8}$, scanning the $\beta$-parameter
from \eq{IBeq2} from $10^{-10}$ to $1$ in 20,000 logarithmically equispaced steps. Our method is seen to improve on 
the BA method in two ways. First, our method finds the entire continuous frontier, whereas the BA method finds only six discrete disconnected points. This is because the BA-method tries to maximize the the DIB-objective from \eq{IBeq2} and thus cannot discover points where the Pareto frontier is convex as discussed above.
Second, our method finds the exact frontier, whereas the BA-method finds only approximations, which are seen to generally lie below the true frontier.

\section{Conclusions \& Discussion}
\label{ConclusionsSec}


We have presented a method for mapping out the Pareto frontier for classification tasks (as in \Fig{paretoFig}), reflecting the tradeoff between retained entropy and class information. 
In other words, we have generalized the quest for maximizing raw classification accuracy to that of 
mapping the full Pareto frontier corresponding to the accuracy-complexity tradeoff. The optimal soft classifiers 
that we have studied  (corresponding to points on the Pareto frontier) are useful for the same reason that
the DIB method is useful: {\eg} overfitting less and therefore generalizing better.

We first showed how a random variable $X$ (an image, say) drawn from a class $Y\in\{1,...,n\}$ can be distilled into a vector $W=f(X)\in \mathbb{R}^{n-1}$ losslessly, so that $I(W,Y)=I(X,Y)$.
For the $n=2$ case of binary classification, we then showed how the Pareto frontier is swept out by a one-parameter family of binnings of $W$ into a discrete variable $Z=g_\beta(W)\in\{1,...,m_\beta\}$ that corresponds to binning $W$ into $m_\beta=2,3,...$ bins, such that $I(Z,Y)$ is maximized for each fixed entropy $H(Z)$.
Our method efficiently finds the exact Pareto frontier, significantly outperforming the 
Blahut-Arimoto (BA) method \cite{blahut1972computation,arimoto1972algorithm}.
Our MATLAB code for computing the Pareto frontier is freely available on 
GitHub.\footnote{A MATLAB implementation is available here:\\
\url{https://github.com/tailintalent/distillation}.}

\subsection{Relation to Information Bottleneck}

As mentioned in the introduction, the Discrete Information Bottleneck (DIB) method \cite{strouse2017deterministic}
maximizes a linear combination $I(Z,Y)-\beta H(Z)$ of the two axes in \fig{paretoFig}.
We have presented a method solving a generalization of the DIB problem.
The generalization lies in 
switching the objective  from \eq{IBeq2} to \eq{ParetoDefEq}, which has the advantage of discovering the full Pareto frontier in \fig{paretoFig} instead of merely the corners and concave parts (as mentioned, the DIB objective cannot discover convex parts of the frontier).
The solution lies in our proof that the frontier is spanned by binnings of the likelihood into $2, 3, 4, ...$ bins, which enables it to be computed more efficiently than with the
iterative/variational method of \cite{strouse2017deterministic}.

The popular original Information Bottleneck (IB) method \cite{tishby2000information} generalizes DIB by allowing the compression function $g(X)$ to be non-deterministic, thus adding noise that is independent of $X$. Starting with a Pareto-optimal $Z\equiv g(X)$ and adding such noise will simply shift us straight to the left in \fig{paretoFig}, away from the frontier (which is by definition monotonically decreasing) and into the Pareto-suboptimal region in the $I(Y;Z)$ vs. $H(Z)$ plane. As shown in \cite{strouse2017deterministic}, IB-compressions tend to altogether avoid the rightmost part of \fig{paretoFig}, with an entropy $H(Z)$ that never drops below some fixed value independent of $\beta$.


 
 
 \subsection{Relation to phase transitions in DIB learning}

Recent work has revealed interesting phase transitions that occur during information bottleneck learning \cite{chechik2005information,strouse2017deterministic,wu2019learnability,wu2019learnabilityEntropy},
as well as phase transitions in other objectives, e.g. $\beta$-VAE \cite{rezende2018taming}, infoDropout \cite{achille2018emergence}.
Specifically, when the $\beta$-parameter that controls the tradeoff between information retention and model simplicity is continuously adjusted, the resulting point in the IB-plane can sometimes ``get stuck" or make discontinuous jumps.
For the DIB case,  our results provide an intuitive understanding of these phase transitions in terms of the geometry of the Pareto frontier.

Let us consider \fig{paretoAnalyticFig} as an example. 
The DIB maximiziation of $I(Z,Y)-\beta H(Z)$ geometrically corresponds to finding a tangent line of the Pareto frontier of slope $-\beta$. 

If the Pareto frontier $I_*(H)$ were everywhere continuous and concave, so that $I_*''(H)<0$,
then its slope would range from some steepest value $-\beta_*$ 
at the right endpoint $H=0$ and continuously flatten out as we move leftward, asymptotically approaching 
zero slope as $H\to\infty$.
The learnability phase transition studied in \cite{wu2019learnability,wu2019learnabilityEntropy} would then occur when $\beta=\beta_*$:
for any $\beta\ge\beta_*$, the DIB method learns nothing, \eg, discovers as optimal the point 
$(H,I)=(0,0)$ where $Z$ retains no information whatsoever about $Y$.
As $\beta\le\beta_*$ is continuously reduced, the DIB-discovered point would then continuously move up and to the left along the Pareto frontier.

This was for the case of an everywhere concave frontier, but Figures~\ref{paretoAnalyticFig} and~\ref{paretoFig} show that actual Pareto frontiers need {\it not} be concave --- indeed, none of the frontiers that we have computed are concave. Instead, they are seen to consist of long convex segments joint together by short concave pieces near the ``corners".  
This means that as $\beta$ is continuously increased, the DIB solution exhibits first-order phase transitions, making discontinuous jumps from corner to corner at certain critical $\beta$-values; these phase transitions correspond to increasing the number of clusters into which the data $X$ is grouped.

\subsection{Outlook}

Our results suggest a number of opportunities for further work, ranging from information theory to machine learning, neuroscience and physics.

As to information theory, it will be interesting to try to generalize our method from binary classification into classification into more than two classes. Also, one can ask if there is a way of pushing 
the general information distillation problem all the way to bits.
It is easy to show that a discrete random variable $Z\in\{1,...,m\}$ can always be encoded as 
$m-1$ independent random bits (Bernoulli variables) $B_1,...,B_{m-1}\in\{0,1\}$, defined by\footnote{The mapping $z$ from bit strings $\B$ to integers $Z\equiv z(\B)$
is defined so that $z(\B)$ is the position of the last bit that equals one when $\B$ is preceded by a one. For example, for $m=4$, the mapping from length-3 bit strings 
$\B\in\{0,1\}^3$ to integers $Z\in\{1,...,4\}$
is 
$z(001)=z(011)=z(101)=z(111)=4$, 
$z(010)=z(110)=3$,
$z(100)=2$,
$z(000)=1$.
}
\beq{BitProbEq}
P(B_k\hbox{=}1)=P(Z\hbox{=}k+1)/P(Z\le k+1),
\eeq
although this generically requires some information bloat.
So in the spirit of the introduction, is there some useful way of generalizing 
PCA, autoencoders, CCA and/or the method we have presented so that the quantities
$Z_i$ and $Z'_i$ in Table~\ref{ComparisonTable} are not real numbers but bits?

As to neural networks, it is interesting to explore novel classifier architectures that reduce the overfitting and resulting overconfidence revealed by \fig{CondProbFig}, as this might significantly increase the amount of information we can distill into our compressed data. It is important not to complacently declare victory just because classification accuracy is high; as mentioned, even 99\% binary classification accuracy can waste 8\% of the information.


As to neuroscience, our discovery of optimal ``corner" binnings begs the question of whether
evolution may have implemented such categorization in brains.
For example, if some binary variable $Y$ that can be inferred from visual imagery is evolutionarily important for a given species (say, whether potential food items are edible), might our method help predict how many distinct colors $m$ their brains have evolved to classify hues into? In this example, $X$ might be a triplet of real numbers corresponding to light intensity recorded by three types of retinal photoreceptors, and the integer $Z$ might end up corresponding so some definitions of yellow, orange, {\etc}
A similar question can be asked for other cases where brains define finite numbers of categories, for example categories defined by distinct words. 

As to physics, it has been known even since the introduction of Maxwell's Demon that a physical system can use information about its environment to extract work from it. If we view an evolved life form as an intelligent agent seeking to perform such work extraction, then it faces a tradeoff between retaining too little relevant infomation (consequently extrating less work) and retaining too much (wasting energy on information processing and storage). 
Susanne Still recently proved the remarkable physics result \cite{still2017thermodynamic} that the lossy data compression optimizing such work extraction efficiency is precisely that prescribed by the above-mentioned Information Bottleneck method \cite{tishby2000information}. As she puts it, an intelligent data representation strategy emerges from the optimization of a fundamental physical limit to information processing.
This derivation made minimal and reasonable seeming assumptions about the physical system, but did not include an energy cost for information encoding. We conjecture that this can be done such that an extra Shannon coding term proportional to $H(Z)$ gets added to the loss function, which means that when this term dominates, the generalized Still criterion would instead prefer the Deterministic Information Bottleneck or one of our Pareto-optimal data compressions. 

Although noise-adding IB-style data compression may turn out to be commonplace in many biological settings, it is striking that the types of data compression we typically associate with human perception intelligence appears more deterministic, in the spirit of DIB and our work. For example, when we compress visual input into ``this is a probably a cat", we do not typically add noise by deliberately flipping our memory to ``this is probably a dog". Similarly, 
the popular jpeg image compression algorithm dramatically reduces image sizes while retaining essentially all information that we humans find relevant, and does so deterministically, without adding noise.

It is striking that simple information-theoretical principles such as IB, DIB and Pareto-optimality appear relevant
across the spectrum of known intelligence, ranging from extremely simple 
physical systems as in Still's work all the way up to high-level human perception and cognition.
This motivates further work on the exciting quest for a deeper understanding of Pareto-optimal data compression 
and its relation to neuroscience and physics. 

\bigskip

{\bf Acknowledgements:} 
The authors wish to thank Olivier de Weck for sharing the AWS multiobjective optimization software. 
This work was supported by The Casey and Family Foundation, the Ethics and Governance of AI Fund, the Foundational Questions Institute, the Rothberg Family Fund for Cognitive Science  
and by Theiss Research through TWCF grant \#0322.
The opinions expressed in this publication are those of the authors and do not necessarily reflect the views of the funders.

\appendix

 \section{Binning can be practically lossless}
\label{LosslessBinningAppendix}

If the conditional probability distribution $p_1(w)\equiv P(Y\hbox{=}1|W\hbox{=}w)$
is a slowly varying function and the range of $W$ is divided into tiny bins, then 
$p_1(w)$ will be almost constant within each bin and so binning $W$ (discarding information about the exact position of $W$ within a bin) should destroy almost no information about $Y$.
This intuition is formalized by the following theorem, which says that a random variable $W$
 can be binned into a finite number of bins at the cost of losing arbitrarily little information about $Y$.
\begin{theorem}
\label{LosslessBinningTheorem}
Binning can be practically lossless: Given a random variable $Y\in\{1,2\}$ and
a uniformly distributed random variable $W \in [0,1]$
such that the conditional probability distribution 
$p_1(w)\equiv P(Y\hbox{=}1|W\hbox{=}w)$ is monotonic,
there exists for any real number $\epsilon>0$ 
a vector $\b\in\mathbb{R}^{N-1}$ of bin boundaries such 
that the information reduction
$$\Delta I\equiv I[W,Y]  - I[B(W,\b),Y] < \epsilon,$$
where $B$ is the binning function defined by
\eq{BinningFuncDef}.
\end{theorem}

\begin{proof}
The binned bivariate probability distribution is 
\beq{FineBinnedProbEq}
P_{ij}\equiv P(Z\hbox{=}j,Y=i) = \int_{b_{j-1}}^{b_j}p_i(w)dw
\eeq
with marginal distribution
\beq{binnedMarginalEq}
P^Z_j\equiv P(Z\hbox{=}j) = b_j-b_{j-1}.
\eeq
Let $\pbar_i(w)$ denote the piecewise constant function 
that in the $j^{\rm th}$ bin $b_{j-1}<w\le b_j$ takes the average value of $p_i(w)$ 
in that bin, \ie, 
\beq{pbarDefEq}
\pbar_i(w)\equiv {1\over b_j-b_{j-1}}\int_{b_{j-1}}^{b_j}p_i(w)dw={P_{ij}\over P^Z_j}.
\eeq
These definitions imply that
\beq{sumIntegralEq}
-\sum_{j=1}^N P_{ij}\log{P_{ij}\over P^Z_j}=\int_0^1 h\left[\pbar_i(w)\right]dw,
\eeq
where $h(x)\equiv -x\log x$.
Since $h(x)$ vanishes at $x=0$ and $x=1$ and takes its intermediate maximum 
value 
at $x=1/e$, the function 
\beq{hstarDefEq}
h_*(x)\equiv
\left\{
\begin{tabular}{ll}
$h(x)$			&if $x<e^{-1}$,\\
$2h(e^{-1})-h(x)$	&if $x\ge e^{-1}$\\
\end{tabular}
\right.
\eeq
is continuous and increases monotonically for
$x\in [0,1]$,
with $h'_*=|h'(x)|$.
This means that if we define the non-negative monotonic function
$$h_+(w)\equiv h_*[p_1(w)] - h_*[p_2(w)],$$
it changes at least as fast as either of its terms, so that 
for any $w_1$, $w_2\in [0,1]$, we have
\beqa{hChangeBoundEq}
\left|h\left[p_i(w_2)\right] - h\left[p_i(w_1)\right]\right|
&\le&\left|h_*\left[p_i(w_2)\right] - h_*\left[p_i(w_1)\right]\right|\nonumber\\
&\le&|h_+(w_2)-h_+(w_1)|.
\eeqa
We will exploit this bound to limit how much $h\left[p_i(w)\right]$ can vary within a bin.
Since $h_+(0)\ge 0$ and $h_+(1)\le 2h_*(1)=4/e\ln 2\approx 2.12 < 3$,
we pick $N-1$ bins boundaries $b_k$ implicitly defined by
\beq{binPlacementEq}
h_+(b_j)=h_+(0) + [h_+(1)-h_+(0)] {j\over N}
\eeq
for some integer $N\gg 1$.
Using \eq{hChangeBoundEq}, this implies that 
\beq{pbarBoundEq}
\left|h\left[\pbar_i(w)\right] - h\left[p_i(w)\right]\right|\le {h_+(1)-h_+(0)\over N}<{3\over N}.
\eeq

The mutual information between two variables is given by 
$I(Y,U)=H(Y)-H(Y|U)$, where the second term (the conditional entropy is given by the following expressions in the cases that we need: 
\beqa{BinnedConditionalEntropyEq}
H(Y|Z)&=&-\sum_{i=1}^N\sum_{j=1}^2 P_{ij}\log{P_{ij}\over P_i},\\
H(Y|W)&=&-\sum_{i=1}^2 \int_0^1 p_i(w)\log p_i(w)dw\label{HYWeq}.
\eeqa
The information loss caused by our binning is therefore
\beqa{DeltaIeq}
\Delta I&=&I(W,Y)  - I(Z,Y) = H(Y|Z)  - H(Y|W)\nonumber\\
           &=&-\sum_{i=1}^2\left( \sum_{j=1}^N P_{ij}\log{P_{ij}\over P_j^Z}+\int_0^1 h\left[p_i(w)\right]dw\right)\nonumber\\
           &=&\sum_{i=1}^2 \int_0^1 \left(h\left[\pbar_i(w)\right] - h\left[p_i(w)\right]\right)dw\nonumber\\
           &\le&\sum_{i=1}^2 \int_0^1 \left|h\left[\pbar_i(w)\right] - h\left[p_i(w)\right]\right|dw\nonumber\\
           &<&\sum_{i=1}^2 \int_0^1 {3\over N} = {6\over N},
\eeqa
where we used \eq{sumIntegralEq}  to obtain the $3^{\rm rd}$ row and
\eq{pbarBoundEq} to obtain the last row.
This means that however small an information loss tolerance $\epsilon$ we want, 
we can guarantee $\Delta I<\epsilon$ by choosing $N>6/\epsilon$ bins placed according to 
\eq{binPlacementEq}, which completes the proof.
\end{proof}

Note that the proof still holds if the function $p_i(w)$ is not monotonic, as long as the number of times $M$ that it changes direction is finite: in that case, we can simply repeat the above-mentioned binning procedure separately in the $M+1$ intervals where $p_i(w)$ is monotonic, using 
$N>6/\epsilon$ bins in each interval, \ie, a total of $N>6M/\epsilon$ bins.


\section{More varying conditional probability boosts mutual information}


Mutual information is loosely speaking a measure of how far a probability distribution $P_{ij}$ is from being separable,
\ie, a product of its two marginal distributions.\footnote{Specifically, the mutual information is the 
Kullback–Leibler divergence between the bivariate probability distribution and the product of its marginals.}
If all conditional probabilities for one variable $Y$ given the other variable $Z$ are identical, then the distribution is separable and the mutual information $I(Z,Y)$ vanishes, so one may intuitively expect that 
making conditional probabilities more different from each other will increase $I(Z,Y)$.
The following theorem formalizes this intuition in a way that enables Theorem~\ref{ContiguousBinningTheorem}.
\begin{theorem}
\label{informationTheorem}
Consider two discrete random variables $Z\in\{1,...,n\}$ and $Y\in\{1,2\}$ 
and define $P_i\equiv P(Z=i)$,\\
 $p_i\equiv P(Y=1|Z=i)$, so that the joint probability distribution
$P_{ij}\equiv P(Z=i,Y=j)$ is given by\\
$P_{i1} = P_i p_i$, $P_{i2} = P_i (1- p_i)$.
If two conditional probabilities $p_k$ and $p_l$ differ, then we increase the mutual information 
$I(Y,Z)$ if we bring them further apart by adjusting $P_{kj}$ and $P_{lj}$ in such a way that both marginal distributions remain unchanged.
\end{theorem}

\begin{proof}
The only such change that keep the marginal distributions for both $Z$ and $Y$ unchanged takes the form 
$$
\left(
\begin{tabular}{llllll}
$P_1 p_1$		&$\cdots$		&$P_k p_k - \epsilon$			&$\cdots$		&$P_l p_l+ \epsilon$		&$\cdots$		\\
$P_1(1 - p_1)$		&$\cdots$		&$P_k(1 - p_k) + \epsilon$		&$\cdots$		&$P_l(1 - p_l) - \epsilon$		&$\cdots$		
\end{tabular}
\right)
$$
where the parameter $\epsilon$ that must be kept small enough for all probabilities to remain non-negative.
Without loss of generality, we can assume that $p_k<p_l$, so that we make the conditional probabilities
\beqa{conditionalProbEq}
P(Y=1|Z=k)={P_{k1}\over P_k} = p_k - \epsilon/P_k,\\
P(Y=1|Z=l)={P_{l1}\over P_l} = p_l + \epsilon/P_l
\eeqa
more different when increasing $\epsilon$ from zero.
Computing and differentiating the mutual information with respect to $\epsilon$, most terms cancel and we find that 
\beq{IderivativeEq}
{\partial I(Z,Y) \over\partial\epsilon} \bigg\rvert_{\epsilon=0}= 
\log\left[
{1/p_k-1  \over 1/p_l - 1}
\right] > 0
\eeq
which means that adjusting the probabilities with a sufficiently tiny $\epsilon>0$ will increase the mutual information, completing the proof.
\end{proof}
 

\bibliography{distillation}

\end{document}